\newif\ifproofs
\proofstrue 

\newif\ifextended

\documentclass[journal]{IEEEtran}
\usepackage{mathtools, amsthm, amsmath,amssymb,amsfonts, dsfont}
\usepackage{textcomp, acronym, cite, url, comment, footnote}
\usepackage{graphicx, xcolor, enumitem, xspace}
\usepackage{multirow}
\usepackage[bookmarks,colorlinks]{hyperref}
\usepackage[linesnumbered,ruled,vlined]{algorithm2e}
\setlist[description]{font=\normalfont}

\def\BibTeX{{\rm B\kern-.05em{\sc i\kern-.025em b}\kern-.08em
    T\kern-.1667em\lower.7ex\hbox{E}\kern-.125emX}}

\let\oldnl\nl
\newcommand{\nonl}{\renewcommand{\nl}{\let\nl\oldnl}}
\newcommand{\lm}{Laplace mechanism\xspace}
\newcommand{\tm}{multivariate $t$-distribution mechanism\xspace}
\newcommand{\indep}{\perp \!\!\! \perp}

\newcommand{\cD}{\mathcal{D}}
\newcommand{\cL}{\mathcal{L}}
\newcommand{\cM}{\mathcal{M}}
\newcommand{\cP}{\mathcal{P}}

\newcommand{\cX}{\mathcal{X}}
\newcommand{\cO}{\mathcal{O}}

\newcommand{\R}{\mathbb{R}}
\newcommand{\C}{\mathbb{C}}
\newcommand{\N}{\mathbb{N}}
\newcommand{\Z}{\mathbb{Z}}
\newcommand{\D}{\mathds{D}}
\newcommand{\E}{\mathds{E}}

\newcommand{\vw}{\boldsymbol{w}}
\newcommand{\vG}{\boldsymbol{G}}
\newcommand{\vx}{\boldsymbol{x}}
\newcommand{\vX}{\boldsymbol{X}}
\newcommand{\vz}{\boldsymbol{z}}
\newcommand{\vl}{\boldsymbol{l}}
\newcommand{\vd}{\boldsymbol{d}}
\newcommand{\ve}{\boldsymbol{e}}
\newcommand{\vn}{\boldsymbol{n}}
\newcommand{\vt}{\boldsymbol{t}}
\newcommand{\vc}{\boldsymbol{c}}
\newcommand{\vh}{\boldsymbol{h}}
\newcommand{\vv}{\boldsymbol{v}}

\newcommand{\vg}{\boldsymbol{g}}

\DeclareMathOperator*{\sign}{\text{sign}}
\DeclareMathOperator*{\argmin}{arg\,min}
\DeclareMathOperator*{\lap}{\text{Laplace}}

\DeclareMathOperator*{\var}{\text{Var}}
\DeclarePairedDelimiter\abs{\lvert}{\rvert}

\DeclareMathOperator{\tr}{tr}

\acrodef{ml}[ML]{machine learning} 
\acrodef{fl}[FL]{federated learning}
\acrodef{ldp}[LDP]{local differential privacy}
\acrodef{dp}[DP]{differential privacy}
\acrodef{dq}[DQ]{dithered quantization}
\acrodef{sdq}[SDQ]{subtractive dithered quantization} 
\acrodef{sgd}[SGD]{stochastic gradient descent}
\acrodef{jopeq}[JoPEQ]{joint privacy enhancement and quantization}
\acrodef{fa}[FedAvg]{federated averaging} 
\acrodef{ppn}[PPN]{privacy preserving noise} 
\acrodef{iid}[i.i.d.]{independent and identically distributed} 
\acrodef{rv}[RV]{random variable}
\acrodef{mlp}[MLP]{multi-layer perceptron}
\acrodef{cnn}[CNN]{convolutional neural network}
\acrodef{snr}[SNR]{signal-to-noise ratio}
\acrodef{dlg}[DLG]{deep leakage from gradients}
\acrodef{idlg}[iDLG]{improved deep leakage from gradients}
\acrodef{ssim}[SSIM]{structural similarity index measure}
\acrodef{mse}[MSE]{mean squared error}
\acrodef{mvu}[MVU]{Minimum Variance Unbiased}

\newtheorem{theorem}{Theorem}[section]
\newtheorem{definition}[theorem]{Definition}
\newtheorem{corollary}{Corollary}

\newtheorem{lemma}{Lemma}

\newcommand{\figSpace}{\vspace{-0.3cm}}

\begin{document}
\title{Joint Privacy Enhancement and Quantization in Federated Learning}

\author{
	\IEEEauthorblockN{Natalie Lang, Elad Sofer, Tomer Shaked, and Nir Shlezinger\\
	} 
	\thanks{Parts of this work were presented at the 2022 IEEE International Symposium on Information Theory as the paper \cite{lang2022joint}. The authors are with the School of ECE, Ben-Gurion University of the Negev, Be'er-Sheva, Israel (e-mails: \{langn, eladsofe, tosha\}@post.bgu.ac.il; nirshl@bgu.ac.il).  	
	}
	\vspace{-0.5cm}
}

\maketitle

\begin{abstract}
\Ac{fl} is an emerging paradigm for training machine learning models using possibly private data available at edge devices. The distributed operation of FL gives rise to challenges that are not encountered in centralized machine learning, including the need to preserve the privacy of the local datasets, and  the communication load due to the repeated exchange of updated models. These challenges are often tackled individually via techniques that induce some distortion on the updated models, e.g., \ac{ldp} mechanisms and lossy compression. In this work we propose a method coined {\em \ac{jopeq}}, which jointly implements lossy compression and privacy enhancement in \ac{fl} settings. In particular, \ac{jopeq} utilizes vector quantization based on random lattice, a universal compression technique whose byproduct distortion is statistically equivalent to additive noise. This distortion is leveraged to enhance privacy by augmenting the model updates with dedicated multivariate \acl{ppn}. We show that \ac{jopeq} simultaneously quantizes data according to a required bit-rate while holding a desired privacy level, without notably affecting the utility of the learned model. This is shown via analytical \ac{ldp} guarantees, distortion and convergence bounds derivation, and numerical studies. Finally, we empirically assert that \ac{jopeq} demolishes common attacks known to exploit privacy leakage.
\end{abstract}
\acresetall 

\section{Introduction}
The unprecedented success of deep learning highly relies on the availability of data, often gathered by edge devices such as mobile phones, sensors, and vehicles. As such, data may be private, and there is a growing need to avoid leakage of private data while still being able to use it for training neural networks. {\em\Ac{fl}}~\cite{mcmahan2017communication,kairouz2021advances, li2020federated, gafni2021federated} is an emerging paradigm for training on edge devices, exploiting their computational capabilities \cite{chen2019deep}. \ac{fl} avoids sharing the users' data, as training is performed locally with periodic centralized aggregations of the models orchestrated by a server.

Learning in a federated manner is subject to several core challenges that are not encountered in traditional centralized \acl{ml} \cite{gafni2021federated,li2020federated}.
These include a repeated  exchange of highly parameterized models between the server and the devices, possibly over  rate-limited channels, notably loading the communication infrastructure and often resulting in considerable  delays \cite{chen2021communication}.  
An additional challenge stems from the need to guarantee that the exchanged models preserve privacy with respect to the local datasets. It was recently shown that while learning on the edge does not involve data sharing, one can still extract private information, or even reconstruct the raw data, from the exchanged models updates, if these are not properly protected \cite{zhu2020deep,zhao2020idlg,huang2021evaluating,yin2021see}.

Various methods have been proposed to tackle the above challenges:
The communication overhead is often relaxed by reducing the volume of the model updates via lossy compression. 
This can be achieved by having each user transmit only part of its model updates by sparsifying or sub-sampling \cite{lin2017deep,alistarh2018convergence,han2020adaptive,konevcny2016federated, hardy2017distributed, aji2017sparse}.
An alternative approach discretizes the model updates via quantization, such that it is conveyed using a small number of bits \cite{alistarh2017qsgd,reisizadeh2020fedpaq,bernstein2018signsgd,shlezinger2020uveqfed,karimireddy2019error,horvath2019natural}.
As for privacy preservation, the \ac{ldp} framework is commonly adopted. \ac{ldp} quantifies privacy leakage of a single data sample when some function of the local datasets, e.g., a trained model, is publicly available \cite{kim2021federated}. 
\ac{ldp} can be boosted by corrupting the model updates with \ac{ppn} \cite{wei2020federated}, via splitting/shuffling \cite{sun2020ldp} or dimension selection \cite{liu2020fedsel}, and by exploiting the noise induced when communicating over a shared wireless channel \cite{seif2020wireless, liu2020privacy}. Prior works also studied the trade-offs between user privacy, utility, and transmission rate; providing utility \cite{kim2021federated} and convergence \cite{liu2022loss} bounds.

Several recent studies consider both challenges of compression and privacy in \ac{fl}.
The works \cite{lyu2021dp, zhang2022leveraging} quantize the local gradient with a differentially private 1-bit compressor. That is, the probability of each coordinate of the gradients to be encoded into one of two possible dictionary words is designed to satisfy the Gaussian mechanism; thus the communication burden is reduced while \ac{dp} is simultaneously guaranteed. However, these methods utilize fixed 1-bit quantizers, and cannot be configured into adaptable communication budget once available.
In \cite{amiri2021compressive}, the authors combine privacy and compression by converting the distortion induced by random lattice coding into a Gaussian noise which holds \ac{dp}. To do so, they perturb the gradient by Gaussian noise prior to quantization, and the overall procedure then holds \ac{dp} according to composition theorem of \ac{dp}.
The above works consider \ac{dp} enhancements, providing users with privacy guarantees from untruthful adversaries, but fail to so for a potential untrusted \ac{fl} third-party server; as can be guaranteed by \ac{ldp}. 

The recent work \cite{chaudhuri2022privacy} proposed a compression method that holds \ac{ldp}. This scheme, referred to in \cite{chaudhuri2022privacy} as \ac{mvu}, utilizes dithered quantization to first transform the model updates into discrete-valued representations, that are subsequently perturbated to hold \ac{ldp}. Yet, this scheme does not leverage the distortion already induced by quantization to enhance privacy via a joint design, but rather tackle each challenge separately in a cascaded fashion. Furthermore, the existing methods all individually perturb each sample of the model updates, thus not leveraging inter-sample correlations to enhance privacy. The fact that both compression and \ac{ldp} enhancement typically involve the addition of some distortion to the model updates vector via, e.g., quantization or \ac{ppn}, motivates the study of unified  multivariate schemes for jointly boosting \ac{ldp} and compression while maintaining the system utility in \ac{fl}.

In this work we propose a dual-function mechanism for enhancing privacy while compressing the model updates in \ac{fl}. Our proposed \ac{jopeq} utilizes universal vector quantization techniques \cite{zamir1992universal,zamir1996lattice}, building upon their ability to transform the quantization distortion into an additive noise term with controllable variance regardless of the quantized data. We  harness the resulting distortion as means to contribute to \ac{ldp} enhancement, combining it with a dedicated additive \ac{ppn} mechanism. For the latter, we specifically employ the highly useful yet less common approach of multivariate \ac{ppn} \cite{reimherr2019elliptical}, which can be naturally incorporated into established low-distortion vector quantization techniques.
\ac{jopeq} results in the local models recovered by the server simultaneously satisfying both desired \ac{ldp} guarantees as well as bit-rate constraints, and does so without notably affecting the utility of the learned model. This is theoretically validated by both analytical LDP guarantees and convergence bound derivation. These findings are also consistently observed in our numerical study, which considers the federated training of different model architectures.

We design \ac{jopeq} by extending the recently proposed \ac{fl} quantization method  of \cite{shlezinger2020uveqfed}, which employs \ac{sdq} using randomized lattices for the local model weights. \ac{jopeq} combines \ac{sdq} with a low-power \ac{ppn}, carefully designed to yield an output that realizes an established \ac{ldp} mechanism. We consider  the multivariate \tm as well as the common scalar \lm; both result in the local models being simultaneously quantized and private. We prove that the information recovered at the server side rigorously satisfies \ac{ldp} guarantees, and characterize the regimes for which privacy can be achieved based solely on \ac{sdq}, i.e., while adding only a negligible level of artificial noise. Our numerical results demonstrate that \ac{jopeq} achieves a lower level of overall distortion and yields more accurate models compared to using separate independent mechanisms for achieving compression and privacy, as well as to the scheme of \cite{chaudhuri2022privacy}. Furthermore, we empirically demonstrate that \ac{jopeq} is privacy preserving  by demolishing the  \ac{dlg} \cite{zhu2020deep} and \ac{idlg} \cite{zhao2020idlg} model inversion attacks, known to exploit privacy leakage and recover data samples from model updates.

The rest of this paper is organized as follows: Section~\ref{sec:sys_model_prelim} briefly reviews the \ac{fl} system model and  related preliminaries in quantization and privacy. Section~\ref{sec:method} presents \ac{jopeq}, theoretically analyzes its \ac{ldp} guarantees and compression properties, deriving of distortion and convergence bounds. \ac{jopeq} is numerically evaluated in Section~\ref{sec:experiments}, while Section~\ref{sec:conclusions} provides concluding remarks.

Throughout this paper, we use boldface lower-case letters for vectors, e.g., $\vx$, boldface upper-case letters for matrices, e.g., $\vX$, and calligraphic letters for sets, e.g., $\cX$. The stochastic expectation, trace, variance, and $\ell_1,\ell_2$ norms are denoted by $\E\{\cdot\}$, $\tr(\cdot)$, $\var(\cdot)$ and ${\|\cdot\|}_1,\|\cdot\|$, respectively, while $\C$ and $\R$ are the sets of complex and real numbers, respectively.

\section{System Model and Preliminaries}\label{sec:sys_model_prelim}
\ifproofs
In this section we present the system model of \ac{fl}  with  quantization and \ac{ldp} constraints. We begin by recalling some relevant basics in \ac{fl} and quantization in Subsections~\ref{subsec:FL}-\ref{subsec:QFL} respectively, after which we provide \ac{ldp} preliminaries in Subsection~\ref{subsec:privacy}, and formulate our problem in Subsection~\ref{subsec:prebelm_def}.
\fi

\subsection{Federated Learning}\label{subsec:FL}
In \ac{fl}, a server trains a model parameterized by  $\vw\in \R^m$ using data available at a group of $K$ users indexed by $1,\ldots, K$. These datasets, denoted $\cD_1,\dots, \cD_K$, are assumed to be private. Thus, as opposed to conventional centralized learning where the server can use $\cD=\bigcup_{k=1}^K \cD_k$ to train $\vw$, in \ac{fl} the users cannot share their data with the server.  
\ifproofs
Let $F_k(\vw)$ be the empirical risk of a model $\vw$ evaluated over the dataset $\cD_k$. The training goal is  to recover the $m\times 1$ optimal weights vector $\vw^{\rm opt}$ satisfying
\begin{equation}\label{eq:w_opt_def}
    \vw^{\rm opt} = \argmin_{\vw} \left\{F(\vw)\triangleq \sum_{k=1}^K \alpha_k F_k\left(\vw\right)\right\},
\end{equation}
\fi
where the averaging coefficients are typically set to $\alpha_k=\frac{\abs{\cD_k}}{\abs{\cD}}$.

Generally speaking, \ac{fl} involves the distribution of a global model to the users. Each user locally trains this global model using its own data, and sends back the model update~\cite{gafni2021federated}. The users thus do not directly expose their private data as  training  is performed locally. The server then aggregates the models into an updated global model and the procedure repeats iteratively. 

Arguably the most common \ac{fl} scheme is \ac{fa}~\cite{mcmahan2017communication}, where the server updates the global model by averaging the local models. Letting $\vw_t$ denote the global parameters vector available at the server at time step $t$, the server shares $\vw_t$ with the users, who each performs $\tau$ training iterations using its local $\cD_k$ to update $\vw_t$ into $\vw^k_{t+\tau}$. The user then shares with the server the model update, i.e., $\vh^k_{t+\tau} = \vw^k_{t+\tau} - \vw_t$. 
The server in turn sets the global model to be 
\begin{align}\label{eq:fl_update}
    \vw_{t+\tau} \triangleq \vw_t + \sum_{k=1}^K \alpha_k \vh^k_{t+\tau}=\sum_{k=1}^K \alpha_k \vw^k_{t+\tau},
\end{align}
where it is assumed for simplicity that all users participate in each \ac{fl} round. The updated global model is again  distributed to the users and the learning procedure continues. 

\ifproofs
When the local optimization at the users side is carried out using \ac{sgd}, then \ac{fa} specializes the {\em local \ac{sgd}} method~\cite{stich2018local}. In this case, each user of index $k$ sets $\vw_t^k = \vw_t$, and updates its local model via 
\begin{align}\label{eq:sgd}
    \vw^k_{t+1} \xleftarrow{}  \vw^k_{t} -\eta_t \nabla F^{i^k_t}_k\left(\vw^k_t\right),
\end{align}
where $i^k_t$ is the sample index chosen uniformly from $\cD_k$, and $\eta_t$ is the step-size. 
\fi
The fact that \ac{fl} involves the users sharing their updated local models $\vw^k_{t+\tau}$ with the server gives rise to the core challenges in terms of communication overload and privacy considerations. This motivates the incorporation of quantization and privacy enhancement techniques,  discussed in the following subsections.

\subsection{Quantization Preliminaries}\label{subsec:QFL}
Vector quantization is the encoding of a set of continuous-amplitude quantities into a finite-bit representation~\cite{gray1998quantization}. The design of vector quantizers often relies on statistical modelling of the vector to be quantized~\cite[Ch. 23]{polyanskiy2014lecture}, which is likely to be unavailable in \ac{fl} \cite{shlezinger2020uveqfed}. Vector quantizers which are invariant of the underlying distribution are referred to as {\em universal vector quantizers}; a leading approach to implement such quantizers is based on lattice quantization~\cite{zamir1992universal}:
\begin{definition}[Lattice Quantizer]\label{def:lattice_Q}
A lattice quantizer of dimension $L \in \Z^+$ and  generator matrix $\vG\ \in \R^{L\times L}$ maps  $\vx\in \R^L$ into a discrete representation $Q_\cL(\vx)$ by selecting the nearest point in the lattice  $\cL \triangleq \{ \vG\vl: \vl\in\Z^L\}$, i.e., 
\begin{equation}
    Q_\cL(\vx) = \mathop{\arg \min}_{\vz \in \cL}\|\vx-\vz\|. 
\end{equation}
\end{definition}
To apply $Q_{\cL}$ to a vector $\vx \in \R^{ML}$, it is divided into $[\vx_1,\ldots,\vx_M]^T$, and each sub-vector is quantized separately. 
A lattice $\cL$ partitions $\R^L$ into cells centered around the lattice points, where the basic cell is $\cP_0=\{\vx:Q_\cL(\vx)=\mathbf{0}\}$. The number of lattice points in $\cL$ is countable but infinite. Thus, to obtain a finite-bit representation, it is common to restrict $\cL$ to include only  points in a given sphere of radius $\gamma$, and the number of lattice points dictates the number of bits per sample $R$. 
\ifproofs
An event in which the input does not reside in this sphere is referred to as {\em overloading}, and quantizers are typically designed to avoid this~\cite{gray1998quantization}. 
In the special case of $L=1$ with $\vG=\Delta_{\rm Q} >0$, $Q_\cL(\cdot)$ specializes conventional scalar uniform quantization $Q(\cdot)$:
\begin{definition}[Uniform Quantizer]\label{def:scalar_uniform_Q}
A mid-tread scalar uniform quantizer with support $\gamma$ and spacing $\Delta_{\rm Q}$ is defined as
\begin{align}\label{eq:scalar_uniform_Q}
    Q(x) = 
    \begin{cases}			
    \Delta_{\rm Q}\left\lfloor\frac{x}{\Delta_{\rm Q}} + \frac{1}{2}\right\rfloor & \text{if } x<\abs{\gamma}\\
	\sign(x)\cdot \gamma & \text{else},
	\end{cases}
\end{align}
where $R~=~\log_2 \big(2\gamma / \Delta_{\rm Q}\big)$ bits are used to represent $x$.
\end{definition}
\fi

The straightforward application of lattice quantization yields a distortion term  $\ve \triangleq Q_{\cL}(\vx) - \vx$ that is deterministically determined by $\vx$. It is thus often combined with {\em probabilistic quantization} techniques, and particularly with \ac{dq} and \ac{sdq}~\cite{lipshitz1992quantization, gray1993dithered}, defined below:
\begin{definition}[\ac{dq}]\label{def:DQ}
The dithered lattice quantization of $\vx \in \R^L$ is given by 
\begin{align}
    Q^{\rm DQ}_{\cL}(\vx)=Q_\cL(\vx+\vd),
\end{align}
where $\vd$ denotes the dither signal, which is independent of $\vx$ and is uniformly distributed over the basic lattice cell $\cP_0$. 
\end{definition}
\begin{definition}[\ac{sdq}]\label{def:SDQ}
The subtractive dithered lattice quantization of $\vx \in \R^L$ is given by 
\begin{align}
    Q^{\rm SDQ}_{\cL}(\vx)=Q^{\rm DQ}_{\cL}(\vx)-\vd= Q_\cL(\vx+\vd)-\vd.
\end{align}
\end{definition}
A key property of \ac{sdq} stems from the fact that its resulting distortion can be made independent of the quantized value. This arises from the following theorem, stated in \cite{gray1993dithered,zamir1996lattice}:
\begin{theorem}\label{thm:SDQ}
For a set of $L\times 1$ vectors $\{\vx_i\}_{i=1}^M$ within the lattice support, i.e., $\Pr(\|\vx_i\|\leq\gamma)=1$, the distortion vectors $\ve_i \triangleq  Q^{\rm SDQ}_{\cL}(\vx_i) - \vx_i$
are \acs{iid}, uniformly distributed over $\cP_0$ and mutually independent of $\{\vx_i\}_{i=1}^M$.
\end{theorem}
Theorem~\ref{thm:SDQ} implies that when  the quantizer is not {overloaded}, the distortion induced by \ac{sdq} can be effectively modeled as white noise uniformly distributed over $\cP_0$.  

\subsection{Local Differential Privacy Preliminaries}\label{subsec:privacy}
\ifproofs
One of the main motivations for \ac{fl} is the need to preserve the privacy of the users' data. Nonetheless, the concealment of the dataset of the $k$th user, $\mathcal{D}_k$, in favor of sharing the weights trained using this data $\vw_t^k$ was shown to be potentially leaky~\cite{zhu2020deep,zhao2020idlg,huang2021evaluating,yin2021see}. Therefore, to satisfy the privacy requirements of \ac{fl}, initiated privacy mechanisms are necessary~\cite{lyu2021dp}. 
\fi

Considering a users-server setting, privacy is commonly quantified in terms of \ac{dp}~\cite{yang2017survey, abowd2018us} and \ac{ldp}~\cite{kasiviswanathan2011can, wang2020federated}. 
While both provide users with privacy guarantees from untruthful adversaries, the former further assumes a trusted third-party server. As this assumption does not necessarily hold for \ac{fl}, to alleviate the privacy concerns of each user the commonly adopted framework is that of \ac{ldp}, defined as follows: 
\begin{definition}[$\epsilon$-\ac{ldp}~\cite{wang2020comprehensive}]\label{def:LDP}
A randomized  mechanism $\cM$ satisfies $\epsilon$-\ac{ldp} if for any pairs of input values $v,v'$ in the domain of $\cM$ and for any possible output $y$ in it, it holds that
\begin{align}
    \Pr [\cM(v)=y] \leq e^\epsilon \Pr [\cM(v')=y].
\end{align}
\end{definition}
Definition~\ref{def:LDP} can be interpreted as a bundle between stochasticity and privacy: if two different inputs are probable (up to some margin or privacy budget) to be associated with the same algorithm output, then privacy is preserved as each data sample is not uniquely distinguishable. A smaller $\epsilon$ means stronger privacy protection, and vice versa.
A common mechanism to achieve $\epsilon$-\ac{ldp} is based on Laplacian \ac{ppn}. By letting $\lap\left(\mu,b\right)$ be the Laplace distribution with location $\mu$ and scale $b$, the \lm is defined as follows:
\begin{theorem}[\cite{dwork2016calibrating}]\label{thm:lm}
Given any function $f:\D\to\R^d$ where $\D$ is a domain of datasets, the \lm defined as 
\begin{equation}\label{eq:lm}
    \cM^{\rm Laplace}\left(f(x),\epsilon\right)=f(x)+
    {\left[z_1,\dots,z_d\right]}^T,
\end{equation}
is $\epsilon$-\ac{ldp}. In \eqref{eq:lm},  $z_i\overset{\acs{iid}}{\sim}\lap\left(0, \Delta f/\epsilon \right)$,  with 
\begin{equation*}
 \Delta f\triangleq\max_{x,y\in\D} \left\|f(x)-f(y)\right\|.   
\end{equation*} 
\end{theorem}
Theorem \ref{thm:lm} concerns multivariate data, yet uses \acs{iid} univariate Laplace random perturbations, and thus fails to exploit spatial correlations in the data for privacy. In particular, the statistical dependence between different variables in each data sample can be exploited for lowering the overall added distortion while keeping the privacy budget unchanged \cite{reimherr2019elliptical,chaudhuri2011differentially,kifer2014pufferfish,awan2021structure}. This is achieved using high-dimensional privacy preserving mechanisms, which engage multivariate \acp{ppn}.
While the straightforward multivariate Laplacian \ac{ppn} fails to satisfy $\epsilon$-\ac{ldp} when $d>1$ \cite{reimherr2019elliptical},
one can guarantee privacy by introducing  multivariate perturbations obeying the $t$-distribution. In particular, letting $\mathbf{t}^d_\nu(\boldsymbol \mu,\mathbf\Sigma)$ denote the $d$-dimensional $t$-distribution with location, scale matrix, and degrees-of-freedom $\boldsymbol \mu$, $\mathbf\Sigma$, and $\nu$ respectively, such a \ac{ppn} satisfies \ac{ldp} guarantees as stated in the following theorem:
\begin{theorem}[\cite{reimherr2019elliptical}]\label{thm:tm}
Given any function $f:\D\to\R^d$, 
define
\begin{align}\label{eq:TdistMech}
    \cM^{\rm t\text{-}dist}\left(f(x),\epsilon\right)=f(x)+\mathbf{t};\quad
    \mathbf{t}{\sim} \mathbf{t}^d_\nu(\boldsymbol \mu,\mathbf\Sigma).
\end{align}
Then, $\cM^{\rm t\text{-}dist}$ satisfies $\epsilon$-LDP, for
\begin{equation}\label{eq:t_mech_eps}
    \exp(\epsilon)=
    {\left[
    \frac{1+c^2/\nu}{1+(c-\Delta)^2/\nu}
    \right]}^{(\nu+d)^2}, 
\end{equation}
where here $c\triangleq\frac{1}{2}\left(\Delta+\sqrt{\Delta^2+4\nu}\right)$ and 
\begin{equation*}
    \Delta\triangleq\max_{x,y\in\D} \|\mathbf\Sigma^{-1/2}\left(f(x)-f(y)\right)\|.   
\end{equation*}
\end{theorem}

The mapping \eqref{eq:TdistMech} is referred to as~\tm. Both Laplace and multivariate $t$-distribution mechanisms use \ac{ppn} to guarantee privacy. While the former is more commonly used in  \ac{fl}, e.g., \cite{zhao2020local},  the latter further leverages spatial correlation in multivariate data to enable privacy enhancement with lower power perturbations \cite{reimherr2019elliptical}. 

\subsection{Problem Formulation}\label{subsec:prebelm_def}
Our goal is to design a mechanism which jointly meets both quantization and privacy constraints in \ac{fl}. Since users are unlikely to have prior knowledge of the  model parameters distribution, we are interested in methods which are universal. Such schemes can be formulated as mappings of the local updates $\vh_t^k\in \R^d$ at the $k$th user into $\tilde{\vh}_t^k\in \R^d$ available at the server, while meeting the following requirements:
\begin{enumerate}[label={\em R\arabic*}]
    \item \label{itm:ldp} {\em Privacy}: the mapping of $\vh_t^k$ into $\tilde{\vh}_t^k$ must be $\epsilon$-\ac{ldp} with respect to $\mathcal{D}_k$ for a given privacy budget $\epsilon$. 
    \item \label{itm:rate} {\em Compression}: the conveying of $\tilde{\vh}_t^k$ from the user to the server should involve at most $R$ bits per sample. 
    \item \label{itm:universal} {\em Universality}: the scheme must be invariant to the distribution of $\vh_t^k$. 
\end{enumerate}
Notice that by \ref{itm:ldp} we are focusing on achieving \ac{ldp} in each time instance, i.e., for each $t$. This is known to enable privacy enhancement in multi-round \ac{fl} training procedures \cite{sun2020ldp}.

Evidently, requirements \ref{itm:ldp}-\ref{itm:universal} can be satisfied by first adding \ac{ppn} to meet \ref{itm:ldp}, followed by universal quantization to satisfy \ref{itm:rate}-\ref{itm:universal}, as both techniques are invariant to the distribution of $\vh_t^k$.
However, these \ac{fl} quantization and privacy boosting schemes can be modelled as corrupting the weights with some random noise (e.g., Thm.~\ref{thm:SDQ} for \ac{sdq}, Thm.~\ref{thm:lm} for \lm, and Thm.~\ref{thm:tm} for \tm). Thus, using separate mechanisms may result in an overall noise which degrades the accuracy of the trained model beyond that needed to meet \ref{itm:ldp}-\ref{itm:universal}. Based on these observations, in the sequel we study a joint design.

\begin{figure*}
\centering
\includegraphics[width=\textwidth]{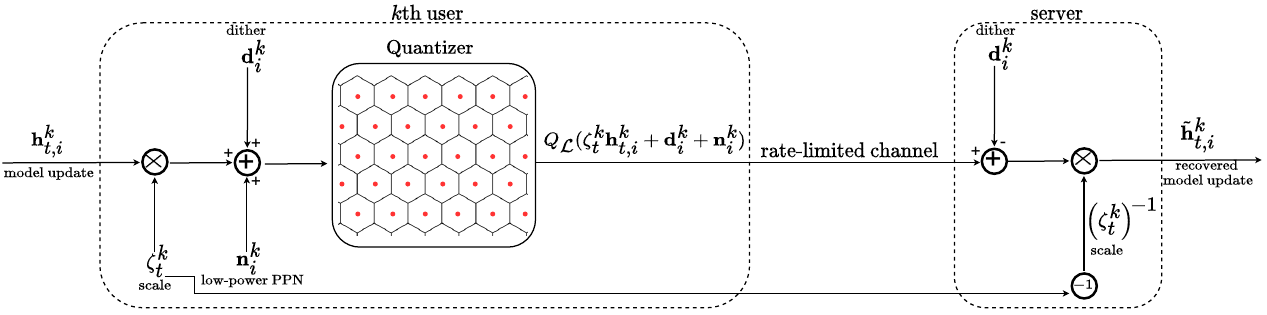}
\caption{Overview of \ac{jopeq}, where the left dashed box represents the encoding by the $k$th user, while the right dashed box describes the server decoding.}
\label{fig:jopeq}
\end{figure*}

\section{Joint Privacy Enhancement and Quantization}\label{sec:method}
In this section we introduce \ac{jopeq}, deriving its steps in Subsection~\ref{subsec:algorithm}, after which  we provide an analysis and discuss its properties in Subsections~\ref{subsec:analysis}-\ref{subsec:discussion}, respectively. 

\subsection{JoPEQ}\label{subsec:algorithm}
We design \ac{jopeq} to exploit the inherent stochasticity of probabilistic quantized \ac{fl} to enhance privacy, relaxing the need to manually introduce possibly dominant \ac{ppn} separately. 
We tackle both \ac{fl} challenges of communication overload and privacy consideration simultaneously by employing the unique distortion properties of \ac{sdq} for implementing established privacy enhancing mechanisms. This is achieved by first corrupting the model updates $\vh_t^k$  with a low-power \ac{ppn} followed by compressing via \ac{sdq}. The \ac{ppn} is carefully designed such that the local model available at the server holds a desirable \ac{ldp} preserving mechanism, e.g., \tm or \lm, yielding an output that is both quantized and privacy preserving.

\ac{jopeq} is divided into two main stages of {\em encoding} and {\em decoding}, carried out by the user and the server, respectively, with a preliminary initialization stage conducted when the \ac{fl} procedure commences. These steps are detailed below and illustrated in Fig.~\ref{fig:jopeq}. As the mechanism description is identical for each of the local users participating in the \ac{fl} process, we henceforth focus only on the $k$th user.

\subsubsection{Initialization}
The first step sets of the privacy budget $\epsilon$, dictated by the application requirements, and the compression scheme parameters, that are based on \cite{shlezinger2020uveqfed}. The latter includes sharing a common seed $s_k$ between each user and the server, that serves as a source of common randomness; fixing the lattice dimension $L$ and its radius $\gamma$; and forming the generator matrix $\mathbf{G}$ (that dictates the lattice $\mathcal{L}$ and its basic cell $\mathcal{P}_0$), which is determined by the bit-rate $R$ \cite[Ch. 2]{conway2013sphere}.

\subsubsection{Encoding}\label{subsec:encoding} The encoding stage takes place at the user sides when the round of local training is finished and the model updates $\vh_t^k$ are ready to be transmitted. The updates are encoded into finite bit representations using a combination of \ac{sdq} and the addition of a carefully designed \ac{ppn} to result with an $\epsilon$-\ac{ldp} mechanism.

{\bf Quantization:}
\ac{jopeq} builds upon the compression scheme of \cite{shlezinger2020uveqfed}, which is implemented via scaling followed by  \ac{sdq} (Def.~\ref{def:SDQ}) with lattice quantization (Def.~\ref{def:lattice_Q}). In particular, this operation involves the scaling  of $\vh^k_t$ by a coefficient $\zeta_t^k$ and its division into $M\triangleq \left\lceil\frac{d}{L}\right\rceil$ distinct $L\times 1$ sub-vectors, applying the $L$-dimensional lattice quantizer on each $\{\zeta_t^k\vh_{t,i}^k\}_{i=1}^M$. The scaling coefficient $\zeta_t^k$ guarantees that the quantizer is unlikely to be overloaded, i.e., that the sub-vectors lie inside the unit $L$-ball with high probability. A candidate setting is thus $\zeta_t^k = \left(\frac{3}{\sqrt{M}}\|\vh^k_t\|\right)^{-1}$, approaching three times the standard deviation of the sub-vectors when they are zero-mean and \acs{iid}, thus assuring no overloading with probability of over $88\%$ by Chebyshev's inequality~\cite{ferentios1982tcebycheff}.  

To implement \ac{sdq}, the user randomizes the dither signals $\vd_i^k\in\R^L$ independently of $\vh^k_{t,i}$ for each $i\in\{1,\ldots,M\}$, where $\vd_i^k$ is uniformly distributed over $\cP_0$. Unlike \cite{shlezinger2020uveqfed}, which considers solely compression, \ac{jopeq} does not directly quantize $\vh^k_{t,i}$, but rather first distorts it with \ac{ppn} $\vn^k_i$ for enhancing privacy. The vectors which are conveyed from the $k$th user to the server are thus $\{Q_\cL(\zeta_t^k\vh^k_{t,i} + \vd^k_i + \vn^k_i)\}_{i=1}^M$ at a bit-rate of at most $R$ bits per sample due the lattice quantizer $Q_\cL(\cdot)$; the associated overhead in conveying the scalar $\zeta_t^k$ is assumed to be negligible compared to conveying $\vh_t^k$. 

{\bf Privacy enhancement:}
The \ac{ppn} signals $\{\vn^k_i\}^M_{i=1}$ are randomized by the $k$th user. Unlike the dither, which the server can also generate with the shared seed, the \ac{ppn} uses a local seed and thus cannot be recreated by the server. 
The \ac{ppn}  $\vn^k_i$ is generated for each $i\in \{1,\dots,M\}$ in an \acs{iid} fashion from a multivariate distribution with the characteristic  function $\Phi_{\vn}:\R^L\mapsto \C$. We propose two settings for $\Phi_{\vn}$ in Theorems~\ref{thm:jopeq_holds_tm}-\ref{thm:jopeq_holds_lm}, for which \ac{jopeq} realizes a \tm and a \lm, respectively.

\subsubsection{Decoding} The server receives the \ac{dq} of $\{\zeta_t^k\vh^k_{t,i}+\vn^k_i\}_{i=1}^M$, which can be modeled as a distorted scaled version of $\{\vh^k_{t,i}\}_{i=1}^M$. Since the server has access to the shared seed  $s_k$ used to generate the dither and also to the scaling coefficient $\zeta_t^k$, this distortion can reduced by implementing \ac{sdq} \cite{gray1993dithered} and re-scaling. This results in the server obtaining $\tilde\vh^k_t$ as the stacking of $\{\tilde\vh^k_{t,i}\}_{i=1}^M$ given by
\begin{align}
\tilde{\vh}^k_{t,i}&=
    (\zeta_t^k)^{-1}\cdot\left( Q_\cL(\zeta_t^k \vh^k_{t,i}+\vn^k_i+\vd^k_i) - \vd^k_i\right)\label{eq:jopeq_middle_result}\\
    &=(\zeta_t^k)^{-1}\cdot Q_\cL^{\rm SDQ}(\zeta_t^k\vh^k_{t,i}+\vn^k_i).\label{eq:jopeq_result}
\end{align}

The decoding procedure of \ac{jopeq} is carried out at the server side, from whom privacy should be preserved. Decoding involves dither subtraction and inverse scaling. Dither subtraction is known to reduce the variance of the overall distortion \cite{gray1993dithered}, and is thus beneficial in terms of improving the accuracy of the global model aggregated via \ac{fa} \cite{shlezinger2020uveqfed}. Yet, since \eqref{eq:jopeq_middle_result} is less distorted, and thus potentially more leaky, than $\{(\zeta_t^k)^{-1}\cdot Q_\cL(\zeta_t^k\vh^k_{t,i}+\vn^k_i+\vd^k_i)\}_{i=1}^M$ due to dither subtraction, we focus our privacy enhancement and analysis on $\tilde\vh^k_t$. The overall algorithm is summarized below as Algorithm~\ref{alg:jopeq}.

\SetKwBlock{Initialization}{Initialization:}{end}
\SetKwBlock{User}{Encode (at the $k$th user side, for each $i$):}{end}
\SetKwBlock{Server}{Decode (at the server side, for each $i$):}{end}
\begin{algorithm}
\caption{\ac{jopeq}}
\label{alg:jopeq}
\KwData{Local model updates at the $k$th user at time step $t$, $\vh^k_t$, decomposed into $\{\vh^k_{t,i}\}_{i=1}^M$\;}
\KwResult{Representation of local model updates $\tilde\vh_t^k$ to be aggregated using \ac{fa}\;}
\nonl  \Initialization{Shared seed $s_k$, lattice $\cL$, and privacy budget $\epsilon$\;}
\nonl \User{Randomize the dither signal $\vd^k_i$, uniformly distributed over $\cP_0$, using seed $s_k$\;
Randomize the \ac{ppn} signal $\vn^k_i$, obeying the characteristic function $ \Phi_{\vn}$\;
Convey $Q_\cL(\zeta_t^k \vh^k_{t,i}+\vd^k_i+\vn^k_i)$ and $\zeta_t^k$ to server\;}
\nonl \Server{Recreate the dither signal $\vd^k_i$ using seed $s_k$\;
Compute $\tilde\vh^k_{t,i}$ using \eqref{eq:jopeq_result}\;}
\end{algorithm}

\subsection{Analysis}\label{subsec:analysis}
Next, we analyze the performance of \ac{jopeq}, characterizing its privacy and compression guarantees, after which we study its associated distortion and convergence properties.

\subsubsection{Privacy}
The encoding procedure of \ac{jopeq} is designed to jointly support privacy by adding \ac{ppn} and compression via lattice quantization. When scaling by $\zeta_t^k$ yields vanishing overloading probability, the setting of the \ac{ppn} characteristic function can guarantee that $\tilde\vh^k_{t,i}$ is equivalent to the model updates $\vh^k_{t,i}$ corrupted by a desired distortion profile. This is stated in the following theorem:
\begin{theorem}\label{thm:jopeq_holds_tm}
Given a fixed $\zeta_t^k$ that yields vanishing overloading probability, let the \ac{ppn} be generated with the characteristic function $\Phi_{\vn}(\vt)$, which for all $\vt\in\R^L$ is defined as:
\begin{multline}\label{eq:ppn_high_dim_tm}
    \Phi_{\vn}(\vt)=
    {\left(
    \int_{\cP_0}
    \frac{\cos\left({\vt}^T\ve\right)}{\abs{\cP_0}}d\ve\right)}^{-1}\times\\
    \exp(i\vt^{T}\boldsymbol{\mu})
    \frac{
    \|\sqrt{\nu}\boldsymbol{\Sigma}^{1/2}\boldsymbol t\|^{\nu/2}}
    {2^{\nu/2-1}\Gamma(\nu/2)}K_{\nu/2}
    \left(\|\sqrt{\nu}\boldsymbol{\Sigma}^{1/2}\boldsymbol t\|\right),
\end{multline}
where $\Gamma(\cdot), K(\cdot)$ are the gamma and 
the third kind modified Bessel functions respectively. 
Then, the distortion at the output of \ac{jopeq}, $\tilde{\vh}^k_{t,i} - \vh^k_{t,i}$,
is mutually independent of $\vh^k_{t,i}$ and obeys an \acs{iid} $\mathbf{t}^d_\nu(\boldsymbol \mu,\mathbf\Sigma)$ distribution, which satisfies \eqref{eq:t_mech_eps}.
\end{theorem}
\ifproofs	
\begin{IEEEproof}
	The proof is given in Appendix \ref{app:jopeq_holds_tm}. 
\end{IEEEproof}
\smallskip
\fi
\ac{jopeq} can also hold other privacy mechanisms, such as the \lm. The latter is commonly used when working with scalar quantities, instead of the \tm which is more complex yet is capable of achieving lower distortion for the same privacy budget \cite{reimherr2019elliptical}. The necessary adaptation for its implementing \lm by \ac{jopeq} is formulated in the following theorem.
\begin{theorem}\label{thm:jopeq_holds_lm}
Given a fixed $\zeta_t^k$ that yields vanishing overloading probability, let the \ac{ppn} be generated with the characteristic function $\Phi_{\vn}(\vt)$,  $\vt=\left[t_1\dots,t_L\right]^T\in\R^L$,  given by:
\begin{align}\label{eq:ppn_high_dim_lm}
    \Phi_{\vn}(\vt)=
    {\left(
        \int_{\cP_0}
        \frac{\cos\left({\vt}^T\ve\right)}{\abs{\cP_0}}d\ve\right)}^{-1}
    \prod_{l=1}^L\frac{1}{{1+{\left(2 t_l/\epsilon\right)}^2}}.
\end{align}
Then, the distortion at the output of \ac{jopeq}, $\tilde{\vh}^k_{t,i} - \vh^k_{t,i}$, is mutually independent of $\vh^k_{t,i}$ and its entries obey an \acs{iid} $\lap\left(0,{2}/{\epsilon}\right)$ distribution.
\end{theorem}
\ifproofs	
\begin{IEEEproof}
	The proof is given in Appendix \ref{app:jopeq_holds_lm}. 
\end{IEEEproof}
\smallskip
\fi
As shown in Appendices \ref{app:jopeq_holds_tm} and \ref{app:jopeq_holds_lm}, the characteristic functions in \eqref{eq:ppn_high_dim_tm} and \eqref{eq:ppn_high_dim_lm} are respectively obtained by exploiting the mutual independence of the \ac{sdq} error and the quantized value. Hence, our ability to rigorously support privacy is a direct consequence of using universal dithered quantization.

Theorems~\ref{thm:jopeq_holds_tm}-\ref{thm:jopeq_holds_lm} imply that the unified effect of the \ac{ppn} and \ac{sdq} in \ac{jopeq} can implement established \ac{ldp} mechanisms. This is achieved in addition to compressing $\vh^k_t$. Thus, combining Theorems~\ref{thm:jopeq_holds_tm}-\ref{thm:jopeq_holds_lm} with Theorems~\ref{thm:lm}-\ref{thm:tm} guarantees privacy as stated in the following corollary:
\begin{corollary}\label{cor:LDP}
\ac{jopeq} with multivariate \ac{ppn} randomized via either \eqref{eq:ppn_high_dim_tm} or \eqref{eq:ppn_high_dim_lm} is  $\epsilon$-\ac{ldp} with respect to $\mathcal{D}_k$.
\end{corollary}

\subsubsection{Compression}
While Corollary~\ref{cor:LDP} focuses on the privacy guarantees of \ac{jopeq}, Algorithm~\ref{alg:jopeq} also implements compression. In particular, \ac{jopeq} is designed to exploit the distortion induced by quantization, which is dictated by the bit-rate $R$, for privacy enhancement, such that \ac{ppn} is only added to complement the remaining level of perturbation needed for a desired privacy budget $\epsilon$ to be obtained. In fact, in some settings one can meet the privacy guarantees based almost solely on the distortion induced by \ac{sdq}, while using \ac{ppn} with infinitesimally small variance, as stated below:

\begin{theorem}\label{thm:PQ_tradeoffs}
When the lattice generator matrix $\vG$ equals the $L\times L$ identity matrix (up to some scalar factor), then \ac{jopeq} with bit-rate $R$ and with \ac{ppn} of an infinitesimally small variance randomized from \eqref{eq:ppn_high_dim_lm}  is $\epsilon$-\ac{ldp} when
\begin{align}\label{eq:PQ_tradeoffs}
    5\approx\sqrt{24}
    \leq&\gamma\epsilon/2^R.
\end{align}
\end{theorem}
\ifproofs	
	\begin{IEEEproof}
		The proof is given in Appendix \ref{app:PQ_tradeoffs}. 
	\end{IEEEproof}
	\smallskip
\fi
The setting for which Theorem~\ref{thm:PQ_tradeoffs} is formulated, i.e., that $\vG$ is the scaled identity matrix, implies that it uses uniform scalar quantizers, as in, e.g., \cite{alistarh2017qsgd, reisizadeh2020fedpaq}.  While Theorem~\ref{thm:PQ_tradeoffs} rigorously holds for a specific family of quantizers, it reveals how \ac{jopeq} jointly balances its privacy and compression requirements \ref{itm:ldp}-\ref{itm:rate}: the stronger the privacy requirement is (smaller $\epsilon$), the more coarse the quantization (smaller $R$) that can support it without effectively injecting \ac{ppn}. 
 
\subsubsection{Weights Distortion}\label{subsubsec:weights_distortion}
In the end of its pipeline, \ac{jopeq} results with an $\epsilon$-\ac{ldp} mechanism applied to the model updates. Consequently, it inherently induces some distortion, being introduced in the \ac{fl} training process, namely, the model update $\vh^k_t$ is transformed into the distorted version $\tilde\vh^k_{t}$. Accordingly, the global model of \eqref{eq:fl_update}, $\vw_{t+\tau}$, which is the desired outcome of \ac{fa}, is changed into
\begin{align}\label{eqn:GlobalDist}
\tilde\vw_{t+\tau}\triangleq\vw_t + \sum_{k=1}^K \alpha_k\tilde\vh^k_{t+\tau}.    
\end{align}
We next show that, under common assumptions used in \ac{fl} analysis, the effect of the distortion in Theorems~\ref{thm:jopeq_holds_tm}-\ref{thm:jopeq_holds_lm} can be mitigated while recovering the desired $\vw_{t+\tau}$ as $\tilde\vw_{t+\tau}$. Thus, the accuracy of the global learned model can be maintained, despite the excessive distortion induced by \ac{jopeq}.

To begin, set the scaling factor to $\zeta_t^k = \left(\frac{3}{\sqrt{M}}\|\vh^k_t\|\right)^{-1}$ and define $\sigma^2$ as the variance of the \ac{ldp} mechanism, e.g., $\sigma^2=\nu\frac{\tr\left(\mathbf\Sigma\right)}{\nu-2}$ for the \tm. 
Next, we adopt the following assumptions on the local datasets  and on the stochastic gradient vector $\nabla F_k^i\left(\vw\right)$:
\begin{enumerate}[label={\em AS\arabic*},series=assumptions]
    \item \label{itm:heterogenity}
    Each dataset $\cD_k$ is comprised \acs{iid} samples. However, different datasets can be statistically heterogeneous, i.e., arise from different distributions. 
    \item \label{itm:bounded_norm}
    The expected squared $\ell_2$-norm of the  vector $\nabla F_k^i\left(\vw\right)$ in \eqref{eq:sgd} is bounded by some $\xi^2_k > 0$ for all $\vw\in\R^m$.
\end{enumerate}
The statistical heterogeneity in Assumption \ref{itm:heterogenity}  is a common characteristic of \ac{fl} \cite{kairouz2021advances,li2020federated,shlezinger2020communication}. It is  consistent with Requirement \ref{itm:universal}, which does not impose any specific distribution structure on the underlying statistics of the training data.
Such heterogeneity implies that the  loss surfaces can differ between users, hence the dependence on $k$ in Assumption \ref{itm:bounded_norm}, 
often employed in distributed learning studies \cite{shlezinger2020uveqfed,li2019convergence, stich2018local, zhang2012communication}.

We can now bound the distance between the recovered model $\tilde\vw_{t+\tau}$ of \eqref{eqn:GlobalDist} and the desired one $\vw_{t+\tau}$, as stated next:
\begin{theorem}\label{thm:Privacy Error Bound}
When \ref{itm:bounded_norm} holds, the mean-squared distance between $\tilde\vw_{t+\tau}$ and $\vw_{t+\tau}$ satisfies
\begin{multline}\label{eq:Privacy Error Bound}
    \E\left\{\|\tilde\vw_{t+\tau}-\vw_{t+\tau}\|^2\right\}\leq
    9\tau\sigma^2
   \left(\sum_{t'=t}^{t+\tau-1}\eta_{t'}^2\right) \sum_{k=1}^K \alpha_k^2 \xi^2_k, 
\end{multline}
\end{theorem}
\ifproofs	
	\begin{IEEEproof}
		The proof is given in Appendix \ref{app:Privacy Error Bound}. 
	\end{IEEEproof}
	\smallskip
\fi

Theorem \ref{thm:Privacy Error Bound}  implicitly suggests that when  the aggregation is fairly balanced, the recovered model can be made arbitrarily close to the desired one by increasing the number of edge users participating in the \ac{fl} training procedure. 
Taking conventional averaging as an example, where each $\alpha_k=1/K$, we get that \eqref{eq:Privacy Error Bound} decreases as $1/K$.  
Furthermore, if $\max \alpha_k \propto 1/{K^c}$ for some $c>1/2$ , which essentially means that the updated model is dominated by a small part of the participating users, then the distortion vanishes in the aggregation process as $K$ grows. Besides, when the step-size $\eta_t$ gradually decreases, which is known to contribute to the convergence of \ac{fl} \cite{stich2018local, li2019convergence}, it follows from Theorem~\ref{thm:Privacy Error Bound} that  the distortion decreases accordingly, further revealing its effect as the \ac{fl} iterations progress, discussed next. 

\subsubsection{Federated Learning Convergence}\label{subsubsec:fl_convergence}
To study the convergence  of \ac{fa}  with \ac{jopeq}, we further introduce the following assumptions, inspired by \ac{fl} convergence studies in, e.g., \cite{shlezinger2020uveqfed,li2019convergence, stich2018local}:
\begin{enumerate}[resume*=assumptions]
    \item \label{itm:obj_smooth}
    The local objective functions $\{F_k(\cdot)\}^K_{k=1}$ are all $\rho_s$-smooth, i.e., for all $\vw_1, \vw_2 \in \R^m$ it holds that
    \begin{multline*}
        F_k(\vw_1)-F_k(\vw_2) \leq 
        (\vw_1 -\vw_2)^T\nabla F_k(\vw_2)\\
        +\frac{1}{2}\rho_s{\|\vw_1 -\vw_2\|}^2.
    \end{multline*}
     \item \label{itm:obj_convex}
     The local objective functions $\{F_k(\cdot)\}^K_{k=1}$ are all $\rho_c$-strongly convex, i.e., for all $\vw_1, \vw_2 \in \R^m$ it holds that
    \begin{multline*}
        F_k(\vw_1)-F_k(\vw_2) \geq 
        (\vw_1 -\vw_2)^T\nabla F_k(\vw_2)\\
        +\frac{1}{2}\rho_c{\|\vw_1 -\vw_2\|}^2.
    \end{multline*}
\end{enumerate}
Assumptions \ref{itm:obj_smooth}-\ref{itm:obj_convex} hold for a range of objective functions used in \ac{fl}, including $\ell_2$-norm regularized linear regression and logistic regression \cite{shlezinger2020uveqfed}. To proceed, following \ref{itm:heterogenity} and as in \cite{li2019convergence, shlezinger2020uveqfed}, we define the heterogeneity gap,
\begin{align}\label{eq:psi_heterogeneity_gap}
\psi \triangleq F(\vw^{\rm opt})-\sum_{k=1}^K\alpha_k \min_{\vw} F_k(\vw),
\end{align}
where $\vw^{\rm opt}$ is defined in \eqref{eq:w_opt_def}. Notice that \eqref{eq:psi_heterogeneity_gap} quantifies the degree of heterogeneity: if the training data originates from the same distribution, $\psi$ tends to zero as the training size grows, and is positive otherwise.

The following theorem characterizes \ac{jopeq} convergence in  conventional \ac{fl} with  local \ac{sgd} training:
\begin{theorem}\label{thm:FL Convergence}
Set $\varphi \triangleq \tau \max \left(1 , 4\rho_s / \rho_c \right)$ and consider \ac{jopeq}-aided \ac{fl} satisfying 
\ref{itm:heterogenity}-\ref{itm:obj_convex}. Under this setting, local \ac{sgd} with step-size $\eta_t=\frac{\tau}{\rho_c(t+\varphi)}$ for each $t\in\N$ satisfies
\begin{multline}\label{eq:FL_Convergence}
    \E\left\{F(\vw_t)\right\}-F(\vw^{\rm opt})\leq\\
    \frac{\rho_s}{2(t+\varphi)}
    \max\bigg(\frac{\rho^2_c+\tau^2b}{\tau\rho_c},
    \varphi\|\vw_0 - \vw^{\rm opt}\|^2\bigg);
\end{multline}
$$b \triangleq 
    \left(1+36\tau^2\sigma^2\right)\sum_{k=1}^K\alpha^2_k\xi^2_k 
    + 6\rho_s\psi
    + 8(\tau-1)^2\sum_{k=1}^K\alpha_k\xi^2_k,$$
where $\vw^{\rm opt},\ \psi$ are defined in \eqref{eq:w_opt_def}, \eqref{eq:psi_heterogeneity_gap} respectively.
\end{theorem}
\ifproofs	
	\begin{IEEEproof}
		The proof is given in Appendix \ref{app:FL Convergence}. 
	\end{IEEEproof}
	\smallskip
\fi

According to Theorem \ref{thm:FL Convergence}, \ac{jopeq} with local \ac{sgd} converges at a rate of $\cO(1/t)$.
This asymptotic  rate implies that as the number of iterations $t$ progresses,  the  learned  model converges to $\vw^{\rm opt}$ with a difference decaying at the same order of convergence as \ac{fl} with neither privacy  nor compression constraints \cite{stich2018local, li2019convergence}. Nonetheless, it is noted that the need  to compress the model updates and enhance their privacy yields an additive term in the coefficient $b$ which depends on the the privacy level $\epsilon$. 
This implies that \ac{fl} typically converges slower when stricter privacy constraints are imposed.

\subsection{Discussion}\label{subsec:discussion}
The proposed \ac{jopeq} is a dual-function mechanism for enhancing privacy while quantizing the model updates in \ac{fl}. 
It builds upon the usage of randomized lattices as an effective quantization method which results in a distortion acting as additive noise independent of the input to be quantized. Not only is this perturbation is mitigated by averaging in~\ac{fa}, as shown in Theorems~\ref{thm:Privacy Error Bound}-\ref{thm:FL Convergence}, but it can also be harnessed for privacy enhancement, as proved in Corollary~\ref{cor:LDP}. \ac{jopeq} exploits this property and the inherent high-dimensional structure of the model updates to transform probabilistic compression into a multivariate privacy enhancement mechanism. It does so such that beyond the imposed quantization distortion, the excess perturbation exhibited due to privacy considerations (the artificial addition of the \ac{ppn}) is relatively minor, particularly when the quantization distortion is dominant, as indicated by Theorem~\ref{thm:PQ_tradeoffs}. Altogether, \ac{jopeq} satisfies \ref{itm:ldp}-\ref{itm:universal} without notably affecting the utility of the learned model, compared with using separate privacy enhancement and quantization, 
as numerically demonstrated in Section~\ref{sec:experiments}.

\ac{jopeq} is designed to result in $\tilde{\vh}_t^k$ being the updates $\vh_t^k$ corrupted by \acs{iid} perturbations which implement a conventional \ac{ldp} mechanism, i.e., \tm or \lm. Nonetheless, one can consider extending our methodology to alternative high-dimensional privacy mechanisms, such as those based on $\ell_2$-mechanisms  \cite{reimherr2019elliptical,chaudhuri2011differentially,kifer2014pufferfish,awan2021structure}. Alternatively, the extension of \ac{jopeq} to other univariate privacy mechanisms, such as those based on the common Gaussian noise \cite{lyu2021dp} (which rely on a slightly less restrictive definition of \ac{ldp} compared with Def.~\ref{def:LDP}), can also be studied. These are all left for future work. 
Furthermore, \ac{jopeq} is expected to enhance privacy also from external adversaries, due to \ac{ldp} post-processing property \cite{xiong2020comprehensive}. Since adversaries cannot reduce the local updates distortion via subtractive dithering (as they do not share the users' seeds), \ac{jopeq} is likely to provide higher privacy protection from them. Still, since \ac{fl} is motivated by the need to avoid sharing local data with a centralized server, characterizing \ac{ldp} guarantees from external adversaries is left for future work.

\ac{jopeq} is backed by rigorous \ac{ldp} guarantees, and is also empirically demonstrated in Section~\ref{sec:experiments} to demolish attacks designed to exploit privacy leakage. Nonetheless, we can still identify possible privacy-related aspects which one has to account for. First, the distortion induced by \ac{sdq} relies on the using of a shared seed between the server and each user. When the seed is generated by the server, one can envision a scenario in which the seed is maliciously designed to affect the procedure, motivating having the seed $s_k$ generated by the $k$th user rather than by the server. Moreover, our privacy guarantees are characterized assuming vanishing overloading probability, where in practice one may wish to quantize with some small level of overloading. Finally, \ac{jopeq} involves sharing a scaling coefficient which depends on the norm of the complete model update vector. Thus, while the \ac{sdq} output is proved to be privacy preserving, the server may still know the norm of the model updates from the scaling coefficient. These last two aspects can possibly be addressed by applying some non-scaled fixed truncation to the model updates at the users' side, along the probable cost of affecting the overall utility. We leave the study of these cases for future investigation.

\section{Experimental Study}\label{sec:experiments}
In this section we numerically evaluate \ac{jopeq} and compare it to 
other approaches realize quantization and privacy in \ac{fl} in Subsection~\ref{subsec:JoPEQevaluation}. Afterwards, we empirically validate the ability of \ac{jopeq} to mitigate privacy leakage upon model inversion mechanism in Subsection~\ref{subsec:DLG}. 
{The source code used in our experimental study, including all the hyper-parameters, is available online at \url{https://github.com/langnatalie/JoPEQ}.}

\subsection{JoPEQ Evaluation Performance}\label{subsec:JoPEQevaluation}

\subsubsection{MNIST Dataset}\label{subsubsec:JoPEQevaluation_mnist}
We first consider the federated training of a handwritten digit classification model using the MNIST dataset \cite{deng2012mnist}. 
The data, comprised of $28 \times 28$ gray-scale images divided into $60,000$  training examples and $10,000$ test examples, is uniformly distributed among $K=10$ users. In each \ac{fl} iteration, the users train their  models using local \ac{sgd} with learning rate $0.1$. 

We evaluate our framework using three different architectures: a linear regression model; a \ac{mlp} with two hidden layers and intermediate ReLU activations; and a \ac{cnn} composed of two convolutional layers followed by two fully-connected ones, with intermediate ReLU activations and max-pooling layers. All three models use a softmax output layer.

We  compare \ac{jopeq} with five other baselines: 
\begin{description}
    \item[{\em FL}:] vanilla \ac{fl}, without privacy or compression constraints. 
    \item[{\em \ac{fl}+\ac{sdq}}:] \ac{fl} with \ac{sdq}-based compression without privacy. 
    \item[{\em \ac{fl}+Lap}:]  \ac{fl} with \lm, without compression. 
    \item[{\em \ac{fl}+Lap+\ac{sdq}}:] the separate application of \lm and \ac{sdq}.
    \item[{\em \ac{mvu}}:] the scheme of \cite{chaudhuri2022privacy}, which introduces discrete-valued \ac{ldp}-preserving perturbation to the quantized representation of the model update.
\end{description}
We use scalar quantizers, i.e., the lattice dimension is $L=1$. For \ac{sdq}, we set $\gamma= 2R + 1/\epsilon$ which numerically assures non-overloaded quantizers with high probability. 

\begin{figure}
	\includegraphics[width=\columnwidth]{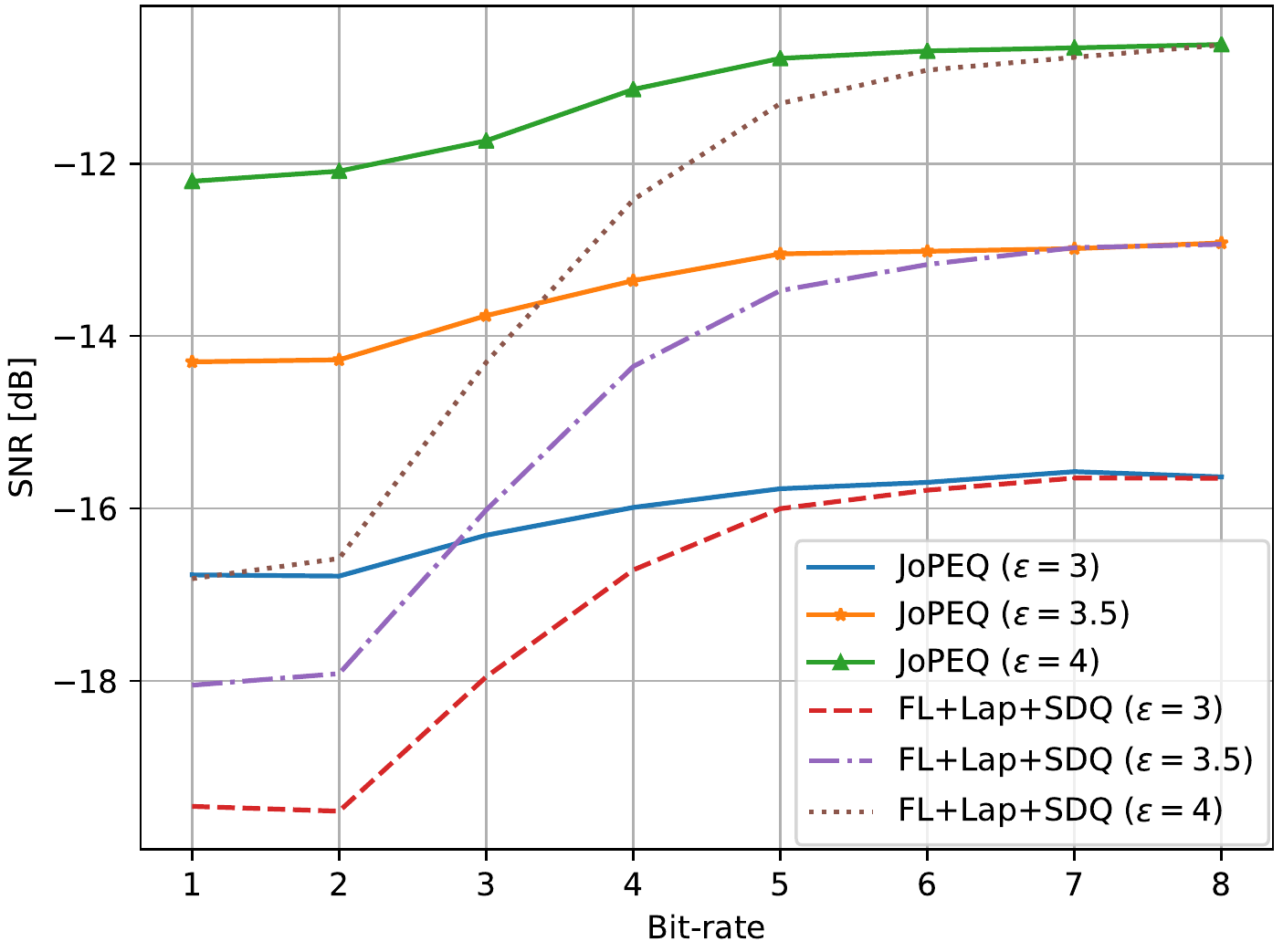}
	\caption{\ac{snr} in received models versus compression bit-rate.}
	\label{fig:SNR}
	\figSpace
\end{figure}

We begin by numerically validating that \ac{jopeq} indeed minimizes the excess distortion compared to individual compression and privacy enhancement operating with the same \ref{itm:ldp}-\ref{itm:rate}. To that aim, we evaluate the average \ac{snr} observed at the server before \ac{fa}, which we compute as the estimated variance of the model weights and divide it by the estimated variance of the distortion, i.e., 
\begin{equation*}
 \text{SNR}\triangleq \frac{1}{K}\sum_{k=1}^K\var(\vh^k_t)/\var(\vh^k_t - \tilde\vh^k_t).   
\end{equation*}
The resulting \ac{snr} values  of \ac{jopeq} are compared with  \ac{fl}+\ac{sdq}+Lap   for different privacy budgets $\epsilon \in\{3,3.5,4\}$  in Fig.~\ref{fig:SNR}. We observe that \ac{jopeq}  yields excess distortion which hardly grows when the bit-rate is decreased, as opposed to separate quantization and \lm. The gains of \ac{jopeq} in excess distortion are thus  dominant in the low bit regime, where quantization induces notable distortion harnessed by \ac{jopeq} for privacy.

Next, we evaluate how the reduced excess distortion of \ac{jopeq}  translates into improved learning. 
\begin{figure}
\includegraphics[width=\columnwidth]{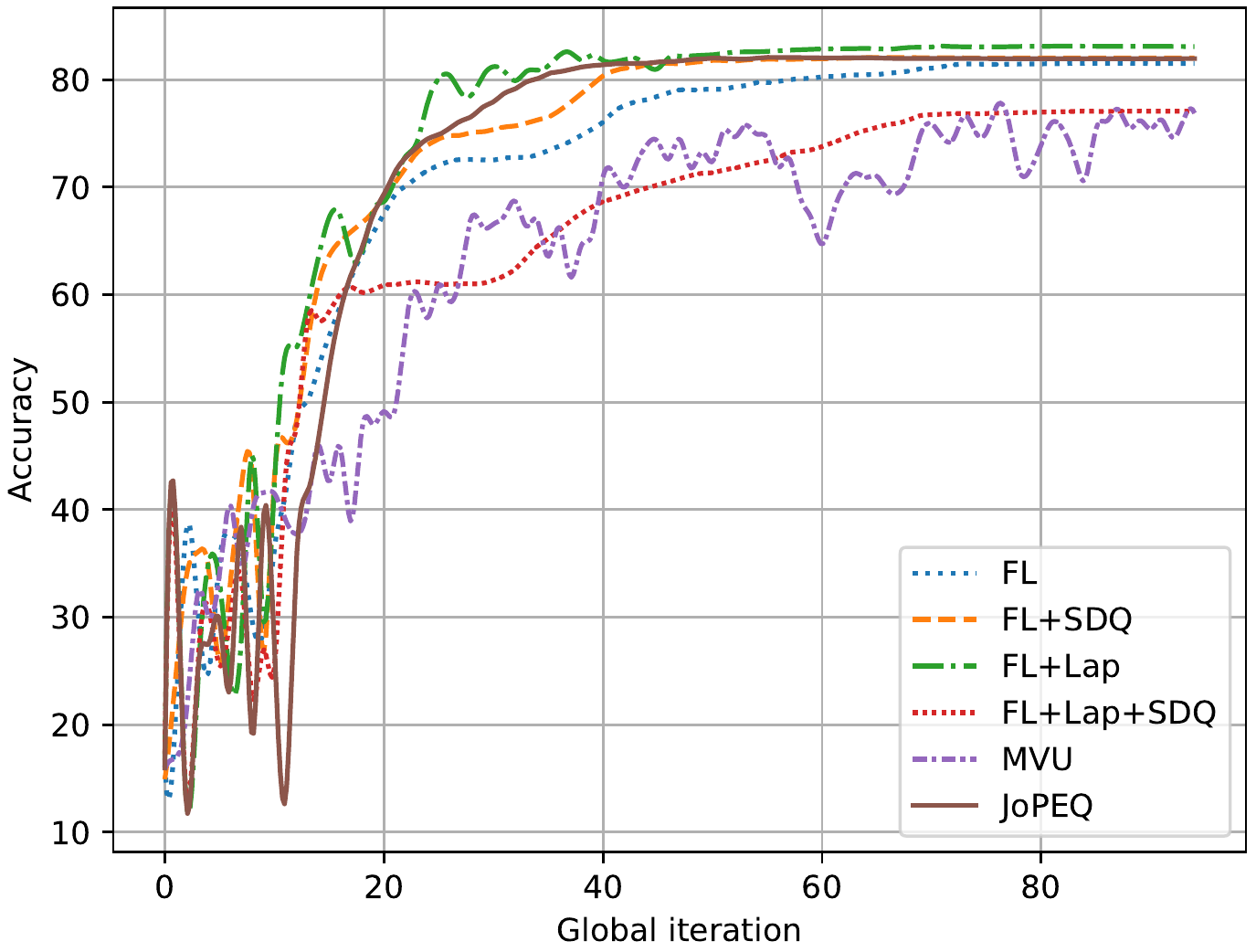}
\caption{Convergence profile of different scalar \ac{fl} schemes.}
\label{fig:learning_curves}
\figSpace
\end{figure}
To that aim, we depict in Fig.~\ref{fig:learning_curves} the validation learning curves for the linear regression model with $R=1$ and $\epsilon=4$. Fig.~\ref{fig:learning_curves} demonstrates that \ac{jopeq} attains almost equivalent performance compared to the alternatives of \ac{fl}+\ac{sdq} and \ac{fl}+Lap (which only meet either \ref{itm:ldp} or \ref{itm:rate}), while simultaneously satisfying both \ref{itm:ldp} and \ref{itm:rate}. We further observe that the disjoint \ac{fl}+Lap+\ac{sdq} as well as the \ac{mvu} scheme of \cite{chaudhuri2022privacy} suffer from excessive distortion as a result of using distinct quantization and privacy mechanism, which deteriorates performance.

\begin{table}
\caption{Test accuracy results for MNIST}
\begin{center}
\begin{tabular}{|c|c|c|c|}
\hline
\cline{2-4} 
 & 
Linear &
MLP & 
CNN\\
\hline
FL & 0.84 & 0.75 & 0.79 \\
FL+SDQ & 0.84 & 0.76 & 0.84 \\
FL+Lap & 0.85 & 0.70 & 0.77\\
FL+Lap+SDQ & 0.78 & 0.55 & 0.75\\
\ac{jopeq} & 0.84 & 0.70 & 0.79\\
\hline
\end{tabular}
\label{tab:architecture_influence}
\end{center}
\end{table}   

In Table~\ref{tab:architecture_influence} we report the baselines test accuracy of the converged models for different examined architectures, which demonstrates that  \ac{jopeq} is beneficial regardless of the model specific design. It is noted that when training deep models, adding a minor level of distortion can improve the converged model, see, e.g., \cite{Guozhong1995NoiseBackprop,sery2021over}. Hence, FL without privacy or compression does not always achieve the best performance. 

\subsubsection{CIFAR-10 Dataset}\label{subsubsec:JoPEQevaluation_cifar10}
\ac{fl} is next implemented for the distributed training of natural image classification model using the CIFAR-10 dataset \cite{CIFAR-10}.
This set is comprised of $32 \times 32$ RGB images divided into $50,000$ training examples and $10,000$ test examples, uniformly distributed among $K=30$ users; each uses local \ac{sgd} with learning rate $0.1$. The  architecture is a \ac{cnn} composed of three convolutional layers followed by four fully-connected ones, with intermediate ReLU activations, max-pooling and dropout layers, and a softmax output layer.

While MNIST was used to validate scalar encoders ($L=1$), here we consider multivariate approaches as baselines:
\begin{description}
    \item[{\em FL}:] vanilla \ac{fl}, without privacy or compression constraints. 
    \item[{\em \ac{fl}+\ac{sdq}$(L=2)$}:] \ac{sdq}-based compression is now used with lattice dimension $L=2$. 
    \item[{\em \ac{fl}+$t$-dist}:] denotes the integration of \tm in \ac{fl} and replaces \ac{fl}+Lap. 
    \item[{\em \ac{fl}+$t$-dist+\ac{sdq}$(L=2)$}:] the separated application of \tm and \ac{sdq}$(L=2)$.
\end{description}
We set the multivariate $t$-distribution parameters to $\boldsymbol \mu = \boldsymbol 0$ and $\mathbf\Sigma = s^2 \boldsymbol I_{2}$ where $\boldsymbol I_{2}$ denotes the $2\times 2$ identity matrix, and $s^2, \nu$ are extracted using \eqref{eq:t_mech_eps} once $\epsilon$ is given. For \ac{sdq}$(L=2)$, we set $\gamma= 1.5 \times \left(1+s^2\nu/(\nu-2)\right)$, numerically assuring low overloading probability. 

\begin{table}
\caption{Test accuracy and SNR results for CIFAR-10}
\begin{center}
\setlength{\tabcolsep}{2.5pt}
\begin{tabular}{|c|c|c|c|c|c|c|}
\hline 
&  \multicolumn{2}{c|}{$R=2, \epsilon=2$} &
\multicolumn{2}{c|}{$R=2, \epsilon=3$} &
\multicolumn{2}{c|}{$R=1, \epsilon=3$} \\
\cline{2-7}
&Acc. & SNR & 
Acc. & SNR & 
Acc. & SNR\\
\hline
FL & 0.73 & $\infty$ & 0.73 & $\infty$ & 0.73 & $\infty$ \\
FL+SDQ$(L=2)$ & 0.72 & 3.53 & 0.72 & 3.53 & 0.72 & -5.34\\
FL+$t$-dist & 0.67 & -23.51 & 0.7 & -15.55 & 0.7 & -15.55\\
FL+$t$-dist+SDQ$(L=2)$ & 0.67  & -21.81   & 0.7 & -15.46 & 0.68 & -17.1\\
\ac{jopeq} & 0.68 & -21.29 & 0.71 & -14.49 & 0.7 & -14.91\\
\hline
\end{tabular}
\label{tab:test_acc_snr_cifar10}
\end{center}
\end{table}   
Table~\ref{tab:test_acc_snr_cifar10} reports the test accuracy and SNR results of the converged baselines for $R\in\{1,2\}$ and $\epsilon\in\{2,3\}$. Comparing the columns of $R=2,\epsilon=2$ and $R=2,\epsilon=3$ demonstrates the privacy budget influence. As expected, higher $\epsilon$ results with less distortion and consequently higher SNR. \ac{fl}+\ac{sdq}$(L=2)$ is invariant to changes in $\epsilon$ values and is the second best, in terms of both accuracy and SNR, after \ac{fl} without compression and privacy considerations, which suffers from no excess distortion.

The two rightmost columns in Table~\ref{tab:test_acc_snr_cifar10} highlight the bit-rate effect. Except for \ac{fl} and \ac{fl}+$t$-dist, that are bit-rate invariant, all baselines show an improvement for $R=2$ compared to $R=1$.
For $R=1, \epsilon=3$ \ac{jopeq} shows the most notable gains compared to the disjoint approach. At the same time, in this extremest bit-rate of $R=1$, the saving in data traffic is the most prominent: in the considered \ac{cnn} architecture for instance, there are $1.7 \cdot 10^5$ learnable weights, that are also the amount of bits need to be transmitted in each iteration, instead of $\approx 10^7$ bits as in the conventional way where each model parameter is represented with double precision.

\subsection{Privacy Leakage Mitigation Evaluation}\label{subsec:DLG}
To empirically validate the ability of \ac{jopeq} to enhance privacy and limit the ability of data leakage attacks, we next assess \ac{jopeq}'s defense performance against model inversion attacks. We consider attacks based on the \ac{dlg} attack mechanism proposed in \cite{zhu2020deep} applied to invert model updates of  neural networks into data samples. In particular, the original \ac{dlg}  generates dummy gradients from dummy inputs, and optimizes the latter to minimize the distance between the former and the real gradients, to yield a successful reconstruction of the training data from the model calculated gradients. 
This approach was enhanced in \cite{zhao2020idlg}, which proposed \ac{idlg}, that  extends \ac{dlg} to separately reconstruct the label beforehand. In order to examine \ac{jopeq} abilities versus challenging attacks, we use \ac{idlg}  and take another step further by supplying the attacker with the ground truth label as part of the input.

To that aim, our setup is the one used in \cite{zhao2020idlg} and includes the same task (image classification), dataset (CIFAR-10), model architecture (LeNet \cite{lecun1989backpropagation}), and optimizer (L-BFGS \cite{liu1989limited}).
In order to evaluate the results and the quality of the reconstructed images compared to the ground truth, we use \ac{ssim} \cite{wang2004image}, which is known to reflect image similarity, in addition to \ac{mse}. We note that the privacy leakage exploited by \ac{dlg} can be applied to a general training procedure involving gradients, regardless of whether or not the latter is conducted in a federated manner. Therefore, to focus solely on the ability of \ac{jopeq} to guarantee privacy, we consider a single user.
In particular, the raw gradients are either forwarded directly to the \ac{idlg} mechanism, or processed beforehand by \ac{jopeq}. Similarly to Subsection~\ref{subsec:JoPEQevaluation}, we also apply \ac{idlg} to gradients distorted by \ac{sdq}-based compression and by \lm.

\begin{figure}
    \centering
    \includegraphics[width=\columnwidth]{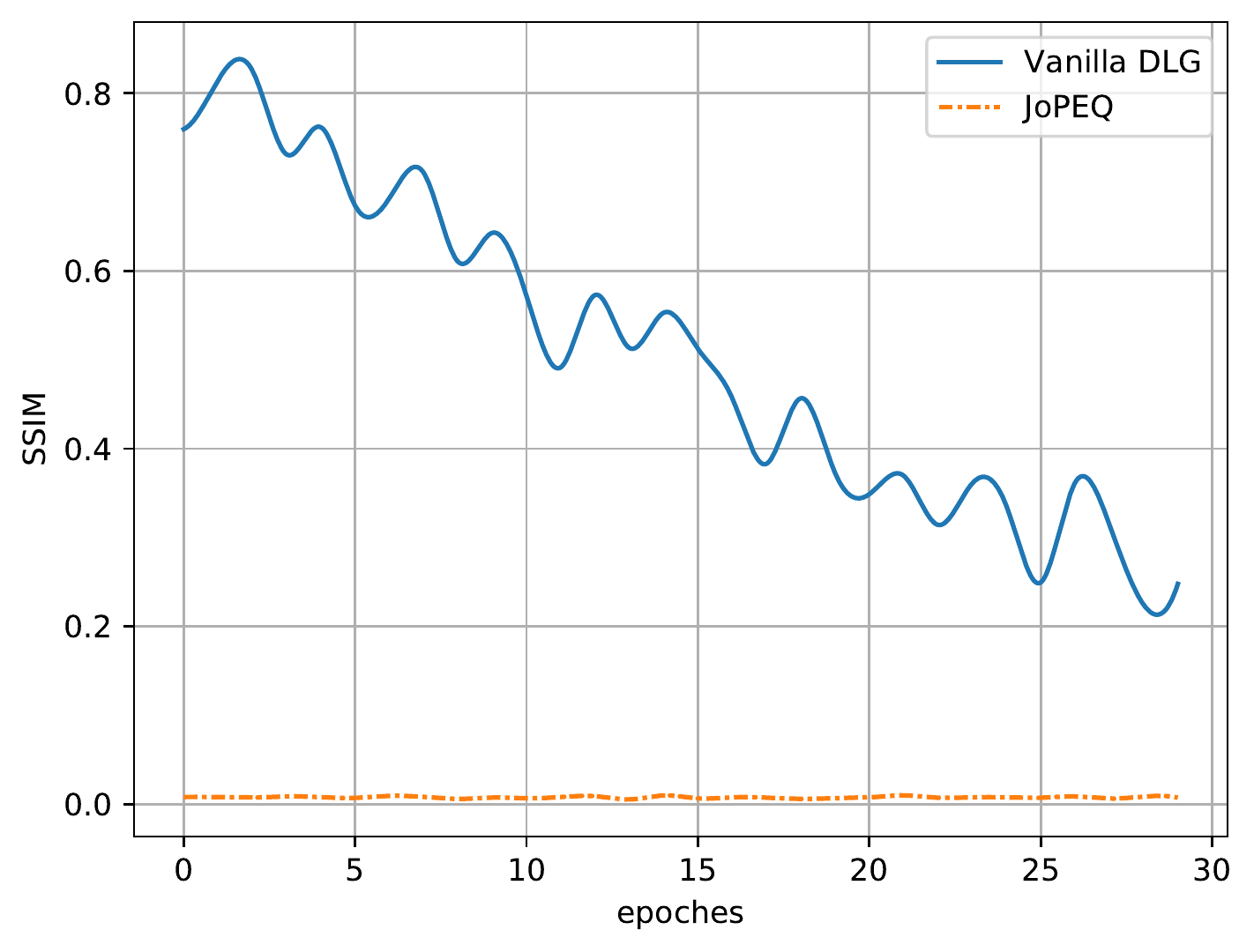}
    \caption{The \ac{ssim} score of vanilla \ac{dlg} and \ac{jopeq} versus epoch number.}
    \figSpace
    \label{fig:ssim_epoch_jopeq_vs_dlg}
\end{figure}

We first validate that the ability of \ac{jopeq} to defend from model inversion attacks does not depend on the model training stage, i.e., it can be activated in any given iteration or epoch number. Fig.~\ref{fig:ssim_epoch_jopeq_vs_dlg} reports the \ac{ssim} score achieved by \ac{idlg} applied directly to the gradients (coined {\em vanilla \ac{dlg}}), through $30$ epochs of training, averaged upon the first $70$ CIFAR-10 images; and the results with the incorporation of \ac{jopeq}, for $\epsilon=3$ and $R=8$. Evidently, vanilla \ac{dlg} struggles more as the training progresses since the gradients tend to be sparser and less informative. However, \ac{jopeq} manages to preserve the privacy  from the first iteration, having its associated \ac{ssim} metric consistently approaches zero. 

\begin{figure}
    \centering
    \includegraphics[width=\columnwidth]{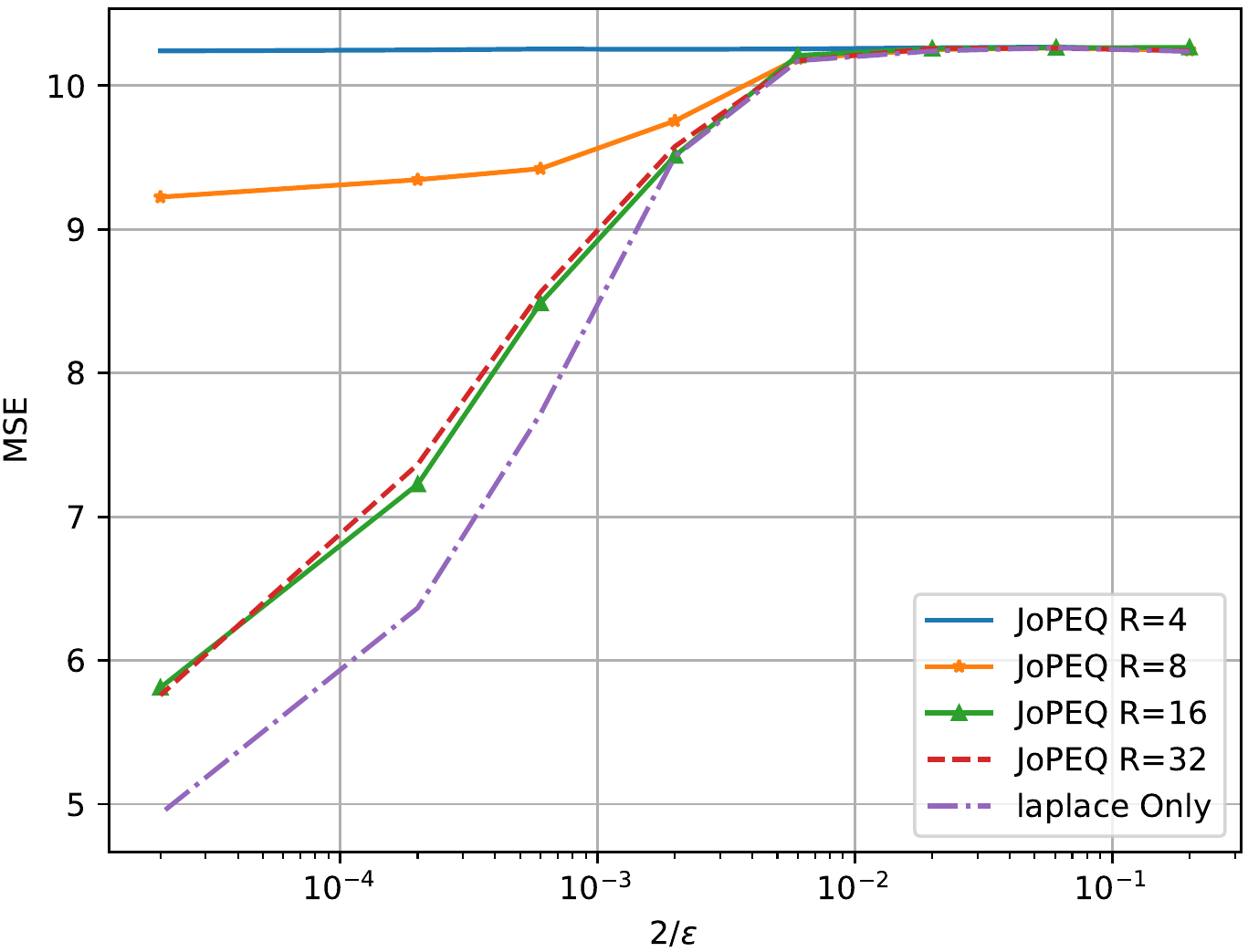}
    \caption{\ac{mse} defense score of \lm and \ac{jopeq} with different bit-raters, for different privacy budgets.}
    \label{fig:JoPEQ_eps_mse_graph}
    \figSpace
\end{figure}
Next, we evlauate the impact of \ac{jopeq}'s bit-rate budget, given a specific privacy budget $\epsilon$, on its privacy preservation performance under \ac{dlg} attack, and also compared that to the reference performance of \lm. The resulting tradeoff between \ac{dlg} reconstruction \ac{mse} and  $2/\epsilon$ are shown in Fig.~\ref{fig:JoPEQ_eps_mse_graph}. 
For fairly loose privacy guarantees, i.e., high $\epsilon$ values, \ac{jopeq} achieves better privacy enhancement compared to \lm, i.e., better than that specified by $\epsilon$, as it leads to inferior reconstruction reflected from the high \ac{mse} values. 
For very strong privacy guarantees, \ac{jopeq} with all examined bit-rates as well as \lm 
successfully demolish the input reconstruction, and result with the highest \ac{mse} value. Furthermore, an interesting interpretation can be given to the curve describing \ac{jopeq} with $R=4$, which saturates even in the regime of high $\epsilon$ values. According to Theorem \ref{thm:PQ_tradeoffs}, the associated privacy enhancement here is effectively achieved by \ac{jopeq} based solely on its compression-induced distortion, and thus does not involve the introduction of additional \ac{ppn}.

\begin{figure}
    \centering
    \includegraphics[width=\columnwidth]{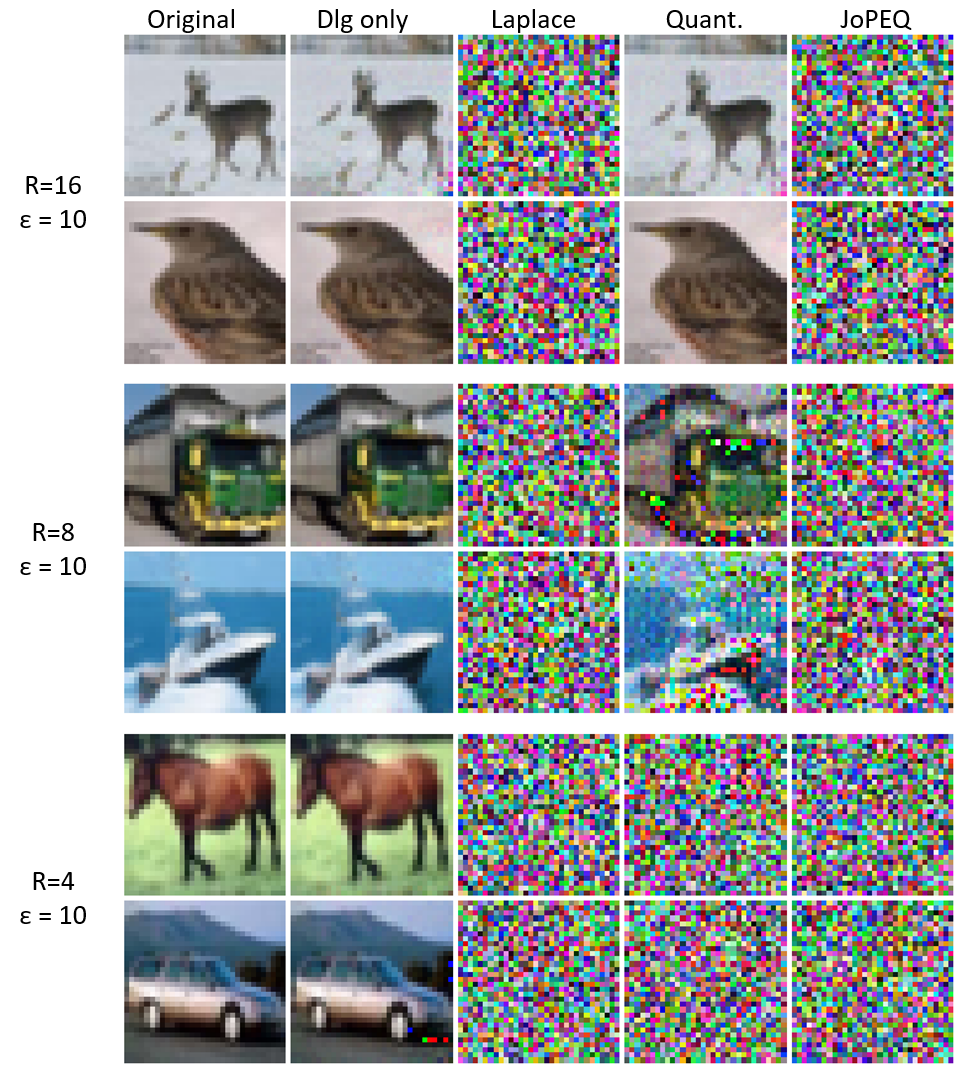}
    \caption{Baselines performances for representative CIFAR-10 examples.}
    \label{fig:DLG_attack_images}
    \figSpace
\end{figure}
Finally, we visually compare between the baselines performances for several representative images in Fig.~\ref{fig:DLG_attack_images}. As expected, vanilla \ac{dlg} results in perfect reconstruction and fails under \lm integration, for privacy budget of $\epsilon=10$. As for \ac{sdq}-based compression, refined quantization holds no privacy defence while a crude one degrades the ability of \ac{dlg} to reconstruct the image used for training. For quantization with $R=8$, the reconstruction is only damaged up to some extent, but \ac{jopeq} with this same bit-rate can finish up the task, and systematically achieves the same ability to demolish model inversion attacks as the \lm. This further assures that \ac{jopeq}, in addition to its being a compression mechanism, is also associated with privacy enhancements. 

\section{Conclusions}\label{sec:conclusions}
We proposed \ac{jopeq}, a unified scheme for jointly compressing and enhancing privacy in \ac{fl}.
\ac{jopeq} utilizes the stochastic nature of \ac{sdq} to boost privacy, combining it with a dedicated \ac{ppn} to yield an exact desired level of distortion satisfying privacy. We proved that \ac{jopeq} can realize different \ac{ldp} mechanisms, and analyzed its associated distortion and convergence, showing that (under some common assumptions) it achieves similar asymptotic convergence profile as \ac{fl} without privacy and compression considerations. We applied \ac{jopeq} for the federated training of different models. We numerically demonstrated that it outcomes with less distorted and more reliable models compared with other applications of compression and privacy in \ac{fl}, while approaching the performance achieved without these constraints, and demolishing common model inversion attacks aim at leaking private date.

\ifproofs
\begin{appendix}
\numberwithin{lemma}{subsection} 
\numberwithin{corollary}{subsection} 
\numberwithin{remark}{subsection} 
\numberwithin{equation}{subsection}	
\subsection{Proof of Theorem \ref{thm:jopeq_holds_tm}}\label{app:jopeq_holds_tm}
\ifextended
To prove the theorem, we first note that by \eqref{eq:jopeq_result} and Theorem \ref{thm:SDQ}, it holds that
\begin{align}
\tilde{\vh}^k_{t,i}&=
    (\zeta_t^k)^{-1}\cdot Q_\cL^{\rm SDQ}(\zeta_t^k\vh^k_{t,i}+\vn^k_i)\notag\\
    &=
    (\zeta_t^k)^{-1}\cdot (\zeta_t^k\vh^k_{t,i}+\vn^k_i + \ve^k_i)\label{eq:jopeq_result_before_scale_app}\\
    &=\vh^k_{t,i}+(\zeta_t^k)^{-1}(\vn^k_i+\ve^k_i)\label{eq:jopeq_result_app}.
\end{align}
Observing \eqref{eq:jopeq_result_before_scale_app}, the private scaled updates $\zeta_t^k\vh^k_{t,i}$ lie in the unit $L$-ball, thus the sensitivity $\left(\Delta\triangleq\max_{x,y\in\D} {\|f(x)-f(y)\|}_2\right)$ is upper bounded by $\sqrt{2}$. This motivates a specific design of $\vn^k_i$ resulting in $\vn^k_i+\ve^k_i$ representing a \tm (Thm.~\ref{thm:tm}). Following \ac{ldp} post-processing property \cite{xiong2020comprehensive}, this also implies $\epsilon$-\ac{ldp} for the non-scaled weights of \eqref{eq:jopeq_result_app}.  To do so, we first note the following lemma which follows from Thm.~\ref{thm:SDQ}:
\begin{lemma}\label{lemma:error_indep_ppn}
For each $i$ and $k$ it holds that $\ve_i^k$ and  $\vn_i^k$ are mutually independent.
\end{lemma}
\begin{proof}
Assume on the contrary that $\ve_i^k\not \indep \vn_i^k$ for some setting $\zeta_t^k\vh^k_{t,i}=\vc$. Since the quantizer is not overloaded, we get from Theorem \ref{thm:SDQ} that $\ve_i^k\indep \vc+\vn_i^k \iff \ve_i^k\indep\vn_i^k$, resulting with a contradiction.
\end{proof}
Next we observe the characteristic function of the overall distortion term. 
Letting $\vt={\left[t_1,\dots,t_L\right]}^T\in\R^L$, the characteristic function is given by
\begin{align}
\notag  \Phi_{\vn_i^k+\ve_i^k}(\vt)
        &=\operatorname{E}
        \left[\exp\left(j\vt^T
        \left(\vn_i^k+\ve_i^k\right)
        \right)\right]\\
\notag  &\stackrel{(a)}{=}
        \operatorname{E}
        \left[\exp
        \left(j\vt^T\vn_i^k
        \right)\right]
        \operatorname{E}
        \left[\exp\left(j\vt^T\ve_i^k\right)\right]\\
        &=\Phi_{\vn_i^k}(\vt)\Phi_{\ve_i^k}(\vt),\label{eq:carac_func_ni+ei}
\end{align}
where $(a)$ follows from Lemma~\ref{lemma:error_indep_ppn}.

Recall that $\ve_i^k$ distributes uniformly  over the basic lattice cell $\cP_0$, which is also symmetric. Consequently, we can write \begin{align}\label{eq:carac_func_ei}
\notag  \Phi_{\ve_i^k}(\vt)&=
        \int_{\cP_0}
        \left(\cos
        \left({\vt}^T
        {\ve'}_i^k\right)
        +j\sin\left(
        {\vt}^T
        {\ve'}_i^k\right)\right)
        \frac{1}{\abs{\cP_0}}d
        {\ve'}_i\\
        &=\int_{\cP_0}
        \frac{\cos\left({\vt}^T
        {\ve'}_i^k\right)}{\abs{\cP_0}}
        d{\ve'}_i^k. 
\end{align}
For \eqref{eq:jopeq_result_before_scale_app} to satisfy the \tm, the overall distortion $\vn^k_i+\ve^k_i$
must obey the multivariate $t$-distribution $\mathbf{t}^d_\nu(\boldsymbol \mu,\mathbf\Sigma)$ which holds \eqref{eq:t_mech_eps}.
Therefore, by \cite{song2014note}, for $\vt \in \R^d$ and $\nu>0$ the characteristic function of $\vn^k_i+\ve^k_i$ satisfies
\begin{align}
\Phi_{\vn^k_i+\ve^k_i}(\vt)
    =\exp(i\vt^{T}\boldsymbol{\mu})
    \frac{
    \|\sqrt{\nu}\boldsymbol{\Sigma}^{1/2}\boldsymbol t\|^{\nu/2}}
    {2^{\nu/2-1}\Gamma(\nu/2)}K_{\nu/2}
    \left(\|\sqrt{\nu}\boldsymbol{\Sigma}^{1/2}\boldsymbol t\|\right).\label{eq:carac_func_noise_is_tdist}
\end{align}
Finally, combining \eqref{eq:carac_func_ni+ei}, \eqref{eq:carac_func_ei} and \eqref{eq:carac_func_noise_is_tdist}, we obtain
\eqref{eq:ppn_high_dim_tm}.  	
\else
To prove the theorem, we first note that by \eqref{eq:jopeq_result} and Theorem \ref{thm:SDQ}, it holds that
\begin{align}
\tilde{\vh}^k_{t,i}&=
    (\zeta_t^k)^{-1}\cdot Q_\cL^{\rm SDQ}(\zeta_t^k\vh^k_{t,i}+\vn^k_i)\notag\\
    &=
    (\zeta_t^k)^{-1}\cdot (\zeta_t^k\vh^k_{t,i}+\vn^k_i + \ve^k_i)\label{eq:jopeq_result_before_scale_app}\\
    &=\vh^k_{t,i}+(\zeta_t^k)^{-1}(\vn^k_i+\ve^k_i)\label{eq:jopeq_result_app}.
\end{align}
Observing \eqref{eq:jopeq_result_before_scale_app}, the private scaled updates $\zeta_t^k\vh^k_{t,i}$ lie in the unit $L$-ball, thus the sensitivity $\left(\Delta\triangleq\max_{x,y\in\D} {\|f(x)-f(y)\|}_2\right)$ is upper bounded by $\sqrt{2}$. This motivates a specific design of $\vn^k_i$ resulting in $\vn^k_i+\ve^k_i$ representing a \tm (Thm.~\ref{thm:tm}). Following \ac{ldp} post-processing property \cite{xiong2020comprehensive}, this also implies $\epsilon$-\ac{ldp} for the non-scaled weights of \eqref{eq:jopeq_result_app}.  To do so, we first note the following straightforward lemma, which follows from Thm.~\ref{thm:SDQ}:
\begin{lemma}\label{lemma:error_indep_ppn}
$\ve_i^k$ and  $\vn_i^k$ are mutually independent $\forall i,k$.
\end{lemma}
Next we observe the characteristic function of the overall distortion term. 
Letting $\vt={\left[t_1,\dots,t_L\right]}^T\in\R^L$, by Lemma~\ref{lemma:error_indep_ppn} the characteristic function is given by
\begin{equation}\label{eq:carac_func_ni+ei}
    \Phi_{\vn_i^k+\ve_i^k}(\vt)=\Phi_{\vn_i^k}(\vt)\Phi_{\ve_i^k}(\vt),
\end{equation}

Recall that $\ve_i^k$ distributes uniformly  over the basic lattice cell $\cP_0$, which is also symmetric. Consequently, we can write \begin{equation}\label{eq:carac_func_ei}
      \Phi_{\ve_i^k}(\vt)=\int_{\cP_0}
        {\abs{\cP_0}}^{-1}{\cos\left({\vt}^T
        {\ve'}_i^k\right)}
        d{\ve'}_i^k. 
\end{equation}
For \eqref{eq:jopeq_result_before_scale_app} to satisfy the \tm, the overall distortion $\vn^k_i+\ve^k_i$
must be distributed as $\mathbf{t}^d_\nu(\boldsymbol \mu,\mathbf\Sigma)$ which holds \eqref{eq:t_mech_eps}.
Therefore, by \cite{song2014note}, for $\vt \in \R^d$ and $\nu>0$ the characteristic function of $\vn^k_i+\ve^k_i$ satisfies
\begin{align}
\Phi_{\vn^k_i+\ve^k_i}(\vt)
    =\exp(i\vt^{T}\boldsymbol{\mu})
    \frac{
    \|\sqrt{\nu}\boldsymbol{\Sigma}^{1/2}\boldsymbol t\|^{\nu/2}}
    {2^{\nu/2-1}\Gamma(\nu/2)}K_{\nu/2}
    \left(\|\sqrt{\nu}\boldsymbol{\Sigma}^{1/2}\boldsymbol t\|\right).\label{eq:carac_func_noise_is_tdist}
\end{align}
Finally, combining \eqref{eq:carac_func_ni+ei}, \eqref{eq:carac_func_ei} and \eqref{eq:carac_func_noise_is_tdist}, we obtain
\eqref{eq:ppn_high_dim_tm}. 
\fi

\subsection{Proof of Theorem \ref{thm:jopeq_holds_lm}}\label{app:jopeq_holds_lm}
The proof outline is identical to that of Appendix \ref{app:jopeq_holds_tm} with the a single difference in the definition of $\Phi_{\vn^k_i+\ve^k_i}(\vt)$. For \eqref{eq:jopeq_result_before_scale_app} to satisfy \lm, the entries of
$\vn^k_i+\ve^k_i\triangleq{\left[n^k_{i,1}+e^k_{i,1},\dots,n^k_{i,L}+e^k_{i,L}\right]}^T$ must obey an \acs{iid} $\lap\left(0,{2}/{\epsilon}\right)$ distribution. Therefore, the characteristic function of $\vn^k_i+\ve^k_i$ satisfies
\ifextended
\begin{align*}
  \Phi_{\vn^k_i+\ve^k_i}(\vt)
        &=\operatorname{E}
        \left[\exp\left(j\vt^T
        \left(\vn^k_i+\ve^k_i\right)\right)\right]\\
        &=\operatorname{E}\left[\prod_{l=1}^L\exp\left(jt_l \left(n^k_{i,l}+e^k_{i,l}\right)\right)\right]\\
        &=\prod_{l=1}^L\operatorname{E}
        \left[\exp\left(jt_l \left(n^k_{i,l}+e^k_{i,l}\right)\right)\right]\\
        &=\prod_{l=1}^L        
        \Phi_{n^k_{i,l}+e^k_{i,l}}\left(jt_l \left(n^k_{i,l}+e^k_{i,l}\right)\right)\\
        &=\prod_{l=1}^L         
        \frac{1}{1+{\left(2t_l/\epsilon\right)}^2},
\end{align*}
\else
\begin{align*}
  \Phi_{\vn^k_i+\ve^k_i}(\vt)
        &=\prod_{l=1}^L        
        \Phi_{n^k_{i,l}+e^k_{i,l}}\left(t_l\right)
        =\prod_{l=1}^L         
        \frac{1}{1+{\left(2t_l/\epsilon\right)}^2},
\end{align*}
\fi
proving the theorem.

\subsection{Proof of Theorem \ref{thm:PQ_tradeoffs}}\label{app:PQ_tradeoffs}
Assuming \ac{jopeq} operates with a lattice quantizer of dimension $L=1$ and holds the \lm, the $l$th entry of $\vn^k_i+\ve^k_i$ satisfies
\begin{align*}
    \var[n^k_{i,l}+e^k_{i,l}]=
    \var\left[\lap\left(\mu=0,b=2/\epsilon\right)\right]=2b^2.
\end{align*}
On the other hand, using Lemma~\ref{lemma:error_indep_ppn},
\begin{align*}
    \var[n^k_{i,l}+e^k_{i,l}]=\var[n^k_{i,l}]+\var[e^k_{i,l}].
\end{align*}
That is, for infinitely small variance PPN, \ac{jopeq} with dynamic rage $\gamma$ and  bit-rate $R$ holds $\epsilon$-\ac{ldp} as long as
\begin{align*}
    2b^2
    \leq& \var[e^k_{i,l}]
    \overset{(a)}{=}\var\left[\mathcal{U}
    \left(\frac{-\Delta_{\rm Q}}{2},
    \frac{\Delta_{\rm Q}}{2}\right)
    \right]={\Delta_{\rm Q}^2}/{12},
\end{align*}
where (a) follows from Theorem~\ref{thm:SDQ} using $L=1$, and $\mathcal{U}(\cdot)$ denotes the continuous uniform distribution. Substituting the fact that $\Delta_{\rm Q} = 2\gamma/2^R$ and $b=2/\epsilon$ we obtain 
\begin{align}
     2(2/\epsilon)^2
    \leq&{(2\gamma/2^R)^2}/{12},
\end{align}
which is equivalent to $\sqrt{24}\leq\varphi\epsilon/2^R$, concluding the proof.

\subsection{Proof of Theorem \ref{thm:Privacy Error Bound}}\label{app:Privacy Error Bound}
\ifextended
Our proof follows a similar outline to that used in \cite{shlezinger2020uveqfed}, with the introduction of additional arguments for handling  privacy constraints, in addition to those of quantization. The unique characteristics of \ac{jopeq}'s error, yield from the subtractive dithered strategy presented in Section \ref{sec:method}, allow us to rigorously incorporate its contribution into the overall proof flow.

In order to bound $\E\left\{\|\tilde\vw_{t+\tau} - \vw_{t+\tau}\|^2\right\}$ let us first express the weights distortion term $\tilde\vw_{t+\tau}-\vw_{t+\tau}$, where $\tilde\vw_{t+\tau}, \vw_{t+\tau}$ are defined in \eqref{eqn:GlobalDist}, \eqref{eq:fl_update} respectively. Similarly to Subsection \ref{subsec:encoding}, we denote by $\{\vv_i\}_{i=1}^M$ the decomposition of a vector $\vv$ into $M$ distinct $L\times 1$ sub-vectors. We have
\begin{align}\label{eq:sub_weights_distortion}
    \tilde\vw_{t+\tau,i}&=\vw_{t,i} + \sum_{k=1}^K \alpha_k\tilde\vh^k_{t+\tau,i}\notag\\
    &=\vw_{t,i} + \sum_{k=1}^K \alpha_k\vh^k_{t+\tau,i} + 
    \sum_{k=1}^K \alpha_k \left(\tilde\vh^k_{t+\tau,i} - \vh^k_{t+\tau,i} \right)\notag\\
    &=\vw_{t+\tau,i} + \sum_{k=1}^K \alpha_k  
    \left(\tilde\vh^k_{t+\tau,i} - \vh^k_{t+\tau,i} \right).
\end{align}
Since $\zeta_t^k = \left(\frac{3}{\sqrt{M}}\|\vh^k_t\|\right)^{-1}$, according to Appendix \ref{app:jopeq_holds_tm} each element in the sum of \eqref{eq:sub_weights_distortion} holds
\begin{align}\label{eq:sub_weights_distortion_explicit}
    \tilde\vh^k_{t+\tau,i} - \vh^k_{t+\tau,i} 
    \overset{\eqref{eq:jopeq_result_app}}{=}
    \frac{\vn^k_i+\ve^k_i}{\zeta_{t+\tau}^k} =
    \frac{3\|\vh^k_{t+\tau}\|}{\sqrt{M}}(\vn^k_i+\ve^k_i),
\end{align}
and consequently
\begin{align}\label{eq:after_cov_incorporation}
        \notag&\E\left\{\|\tilde\vw_{t+\tau} - \vw_{t+\tau}\|^2\right\} \overset{\eqref{eq:sub_weights_distortion}}{=}\\
        \notag&\E\left\{\left\|\sum_{i=1}^M\sum_{k=1}^K \alpha_k\left({{}\tilde\vh^k_{t+\tau}}_{i} - {\vh^k_{t+\tau}}_{i} \right) \right\|^2\right\} \overset{\eqref{eq:sub_weights_distortion_explicit}}{=}\\
        \notag&\E\left\{\left\|\sum_{i=1}^M\sum_{k=1}^K \alpha_k \frac{3}{\sqrt{M}}\|\vh^k_{t+\tau}\|(\vn^k_i+\ve^k_i) \right\|^2\right\} \overset{(a)}{=}\\
        &\notag\E\left\{\E\left\{\left.\left\|\sum_{i=1}^M\sum_{k=1}^K \alpha_k \frac{3}{\sqrt{M}}\|\vh^k_{t+\tau}\|(\vn^k_i+\ve^k_i) \right\|^2 \right\vert \vh^k_{t+\tau} \right\}\right\} \overset{(b)}{=}\\
        \notag&\E\left\{9\sum_{k=1}^K \alpha_k^2\|\vh^k_{t+\tau}\|^2 \E\left\{\left.\left\|(\vn^k_i+\ve^k_i) \right\|^2 \right\vert \vh^k_{t+\tau} \right\}\right\}=\\
        &\E\left\{9\sigma^2\sum_{k=1}^K \alpha_k^2\|\vh^k_{t+\tau}\|^2 \right\},
\end{align}
where (a) follows from the law of total expectation. According to either Theorem \ref{thm:tm} or \ref{thm:lm}, $\{\vn^k_i+\ve^k_i\}$ are \acs{iid} with $\E\left\{\left\|(\vn^k_i+\ve^k_i) \right\|^2\right\} = \sigma^2$ . Therefore, (b) holds by the triangle equality.
Next, if we iterate recursively over \eqref{eq:sgd}, the model update $\vh^k_{t+\tau} = \vw^k_{t+\tau} - \vw_t$ can be written as the sum of the stochastic gradients 
$\vh^k_{t+\tau} = -\sum_{t'=t}^{t+\tau-1} \eta_{t'} \nabla F_k^{i^k_{t'}}\left(\vw^k_{t'}\right)$; where stochasticity stems from the uniformly distributed random indexes $i^k_{t'}$. To utilize this, we first apply the law of total expectation to \eqref{eq:after_cov_incorporation} to yield
\begin{align}\label{eq:app_bounded_wights_distor_norm}
    &\notag\E\left\{\|\tilde\vw_{t+\tau} - \vw_{t+\tau}\|^2\right\}=\\
    &\notag\E\left\{9\sigma^2\sum_{k=1}^K \alpha_k^2
    \E\left\{\left. \|\vh^k_{t+\tau}\|^2 \right\vert \left\{ \vw^k_{t'} \right\} \right\}\right\}=\\
    &\notag\E\left\{9\sigma^2\sum_{k=1}^K \alpha_k^2
    \E\left\{\left. \left\|\sum_{t'=t}^{t+\tau-1} \eta_{t'} \nabla F_k^{i^k_{t'}}\left(\vw^k_{t'}\right) \right\|^2 
    \right\vert \left\{ \vw^k_{t'} \right\} \right\}\right\}  \overset{(a)}{\leq} \\
    &\notag\E\left\{9\sigma^2\sum_{k=1}^K \alpha_k^2 \tau\sum_{t'=t}^{t+\tau-1}\eta_{t'}^2
    \E\left\{\left. \|  \nabla F_k^{i^k_{t'}}\left(\vw^k_{t'}\right) \|^2 
    \right\vert \left\{ \vw^k_{t'} \right\} \right\}\right\}  \overset{(b)}{\leq} \\
    &9\tau\sigma^2 \left(\sum_{t'=t}^{t+\tau-1}\eta_{t'}^2\right) \sum_{k=1}^K \alpha_k^2 \xi^2_k,
\end{align}
where (a) follows from Jensen's inequality $\|\sum_{t'=t}^{t+\tau-1}\vv_t\|^2 \leq \tau\sum_{t'=t}^{t+\tau-1}\|\vv_t\|^2$, viewing the $\ell_2$-norm as a real convex function; and (b) holds since $\E\left\{\|\nabla F_k^{i^k_{t'}}\left(\vw^k_{t'}\right)\|^2\right\} \leq \xi^2_k$ by \ref{itm:bounded_norm}. 
Finally, \eqref{eq:app_bounded_wights_distor_norm} proves the theorem.
\else
Our proof follows a similar outline to that used in \cite{shlezinger2020uveqfed}, with the introduction of additional arguments for handling  privacy constraints, in addition to those of quantization. The unique characteristics of \ac{jopeq}'s error, which follow from the \ac{sdq} strategy presented in Section \ref{sec:method}, allow us to rigorously incorporate its contribution into the overall proof flow.

In order to bound $\E\left\{\|\tilde\vw_{t+\tau} - \vw_{t+\tau}\|^2\right\}$ let us first express the weights distortion term $\tilde\vw_{t+\tau}-\vw_{t+\tau}$, where $\tilde\vw_{t+\tau}, \vw_{t+\tau}$ are defined in \eqref{eqn:GlobalDist}, \eqref{eq:fl_update} respectively. We denote by $\{\vv_i\}_{i=1}^M$ the decomposition of a vector $\vv$ into $M$ distinct $L\times 1$ sub-vectors. It is easy to verify that for $\zeta_t^k = \left(\frac{3}{\sqrt{M}}\|\vh^k_t\|\right)^{-1}$ we have
\begin{equation*}
    \tilde\vw_{t+\tau,i}=\vw_{t+\tau,i} + \sum_{k=1}^K \alpha_k  
    \underbrace{\left(\tilde\vh^k_{t+\tau,i} - \vh^k_{t+\tau,i} \right)}_{=3\|\vh^k_{t+\tau}\|(\vn^k_i+\ve^k_i)/\sqrt{M}}.
\end{equation*}
Consequently,
\begin{align}\label{eq:after_cov_incorporation}
        \notag&\E\left\{\|\tilde\vw_{t+\tau} - \vw_{t+\tau}\|^2\right\}   \overset{(a)}{=}\\
        &\notag\E\left\{\E\left\{\left.\left\|\sum_{i=1}^M\sum_{k=1}^K \alpha_k \frac{3}{\sqrt{M}}\|\vh^k_{t+\tau}\|(\vn^k_i+\ve^k_i) \right\|^2 \right\vert \vh^k_{t+\tau} \right\}\right\} \overset{(b)}{=}\\
        \notag&\E\left\{9\sum_{k=1}^K \alpha_k^2\|\vh^k_{t+\tau}\|^2 \E\left\{\left.\left\|(\vn^k_i+\ve^k_i) \right\|^2 \right\vert \vh^k_{t+\tau} \right\}\right\}=\\
        &\E\left\{9\sigma^2\sum_{k=1}^K \alpha_k^2\|\vh^k_{t+\tau}\|^2 \right\},
\end{align}
where (a) follows from the law of total expectation. According to either Theorem \ref{thm:tm} or \ref{thm:lm}, $\{\vn^k_i+\ve^k_i\}$ are \acs{iid} with $\E\left\{\left\|(\vn^k_i+\ve^k_i) \right\|^2\right\} = \sigma^2$ . Therefore, (b) holds by the triangle equality.
Next, if we iterate recursively over \eqref{eq:sgd}, the model update $\vh^k_{t+\tau} = \vw^k_{t+\tau} - \vw_t$ can be written as the sum of the stochastic gradients 
$\vh^k_{t+\tau} = -\sum_{t'=t}^{t+\tau-1} \eta_{t'} \nabla F_k^{i^k_{t'}}\left(\vw^k_{t'}\right)$; where stochasticity stems from the uniformly distributed random indexes $i^k_{t'}$. To utilize this, we first apply the law of total expectation to \eqref{eq:after_cov_incorporation} to yield
\begin{align}\label{eq:app_bounded_wights_distor_norm}
    &\notag\E\left\{\|\tilde\vw_{t+\tau} - \vw_{t+\tau}\|^2\right\}=\\
    &\notag\E\left\{9\sigma^2\sum_{k=1}^K \alpha_k^2
    \E\left\{\left. \left\|\sum_{t'=t}^{t+\tau-1} \eta_{t'} \nabla F_k^{i^k_{t'}}\left(\vw^k_{t'}\right) \right\|^2 
    \right\vert \left\{ \vw^k_{t'} \right\} \right\}\right\}  \overset{(a)}{\leq} \\
    &\notag\E\left\{9\sigma^2\sum_{k=1}^K \alpha_k^2 \tau\sum_{t'=t}^{t+\tau-1}\eta_{t'}^2
    \E\left\{\left. \|  \nabla F_k^{i^k_{t'}}\left(\vw^k_{t'}\right) \|^2 
    \right\vert \left\{ \vw^k_{t'} \right\} \right\}\right\}  \overset{(b)}{\leq} \\
    &9\tau\sigma^2  \left(\sum_{t'=t}^{t+\tau-1}\eta_{t'}^2\right) \sum_{k=1}^K \alpha_k^2 \xi^2_k,
\end{align}
where (a) follows from Jensen's inequality $\|\sum_{t'=t}^{t+\tau-1}\vv_t\|^2 \leq \tau\sum_{t'=t}^{t+\tau-1}\|\vv_t\|^2$, viewing the $\ell_2$-norm as a real convex function; and (b) holds since $\E\left\{\|\nabla F_k^{i^k_{t'}}\left(\vw^k_{t'}\right)\|^2\right\} \leq \xi^2_k$ by \ref{itm:bounded_norm}. 
Finally, \eqref{eq:app_bounded_wights_distor_norm} proves the theorem.
\fi

\subsection{Proof of Theorem \ref{thm:FL Convergence}}\label{app:FL Convergence}
In the sequel we follow \cite{shlezinger2020uveqfed} and first derive a recursive bound on the weights error, from which we conclude the \ac{fl} convergence bound.
\subsubsection{Recursive Bound on Weights Error}
\ifextended
We first formulate the model update expression, after which we modify that according to the incorporation of \ac{jopeq}.
The local model update is an $m\times 1$ vector given by $\vh^k_{t+1} = \vw^k_{t+1} - \vw^k_t$, which, by \eqref{eq:sgd}, can be written using the stochastic gradient $\vh^k_{t+1} = -\eta_{t} \nabla F_k^{i^k_{t}}\left(\vw^k_{t}\right)$.
On the other hand, every $\tau$ local \ac{sgd} iterations, the $k$th user performs a transmission; that once done with \ac{jopeq} integration, its effect can be modeled as additional noise via
$\tilde{\vh}^k_t=\vh^k_t+\text{err}_t^k$. 
Thus $\text{err}_t^k = \boldsymbol 0$ when $t$ is an integer multiple of $\tau$ and otherwise, by Appendix \ref{app:jopeq_holds_tm}, each $L\times 1$ sub-vector of  $\text{err}_t^k$ scaled by $\zeta_t^k$ obeys the multivariate $t$-distribution $\mathbf{t}^d_\nu(\boldsymbol \mu,\mathbf\Sigma)$ which holds \eqref{eq:t_mech_eps}. The server then sets the new global model as the aggregation of the the edge users updates via \eqref{eq:fl_update}, and send it back to the users who start over the local \ac{sgd}. The overall procedure can be compactly written~as
\begin{align*}
    \notag &\vw^k_{t+1}=\\
    &  \begin{cases*}
            \vw^k_t -\eta_{t}\nabla F_k^{i^k_{t}}\left(\vw^k_{t}\right)+\underbrace{\text{err}_t^k}_{=0} 
            \quad\text{ if } t+1 \neq n\cdot\tau,\ n\in\N,\\
            \sum_{k'=1}^K \alpha_{k'}\left(\vw^{k'}_t -\eta_{t}\nabla F_{k'}^{i^{k'}_t}\left(\vw^{k'}_{t}\right)+\text{err}_t^{k'} \right)
            \quad\text{ else}.
        \end{cases*}
\end{align*}

As in \cite{shlezinger2020uveqfed}, we next define a virtual sequence $\{\vv_t\}$ from $\{\vw^k_t\}$ which can be shown to behave almost like  mini-batch \ac{sgd} wit  batch size $\tau$, while being within a bounded distance of the \ac{fl} model weights $\{\vw^k_t\}$, by properly setting the step-size $\eta_t$. That is, the virtual sequence is given by
\begin{align}\label{eq:virtual_seq}
    \vv_t \triangleq \sum_{k=1}^K \alpha_{k} \vw^k_t,
\end{align}
which coincides with $\vw^k_t$ when $t$ is an integer multiple of $\tau$. We further define the averaged noisy stochastic gradients and the averaged full gradients as
\begin{align}
    \tilde\vg_t &\triangleq \sum_{k=1}^K \alpha_{k}\left(\nabla F_{k}^{i^{k}_t}\left(\vw^{k}_{t}\right)-\frac{1}{\eta_t}\text{err}_t^{k} \right),\\
    \vg_t &\triangleq \sum_{k=1}^K \alpha_{k}\nabla F_{k}\left(\vw^{k}_{t}\right),
\end{align}
respectively. Note that for the specific choice of $\mathbf{t}^d_\nu(\boldsymbol 0,\mathbf\Sigma)$ \ac{jopeq}'s error is zero-mean, and as the sample indexes $\{i^{k}_t\}$ are independent and uniformly distributed, it holds that $\E\{\tilde\vg_t\}=\vg_t$. Additionally, the virtual sequence \eqref{eq:virtual_seq} satisfies $\vv_{t+1}=\vv_t - \eta_t\tilde\vg_t$.

The resulting model is thus equivalent to that used in \cite[App. C]{shlezinger2020uveqfed}, and therefore, by assumptions \ref{itm:obj_smooth}-\ref{itm:obj_convex} it follows that if $\eta_t\leq\frac{1}{4\rho_s}$ then
\begin{multline}\label{eq:expected_distance_bound}
    \E\left\{\left\| \vv_{t+1}-\vw^{\rm opt} \right\|^2\right\}\leq 
    (1-\eta_t\rho_c) \E\left\{\left\| \vv_t-\vw^{\rm opt} \right\|^2\right\} \\
     +6\rho_s\eta^2_t\psi
     + \eta^2_t \E\left\{\left\| \tilde\vg_t-\vg_t \right\|^2\right\} \\
    + 2\E\left\{\sum_{k=1}^K\alpha_{k}\left\|\vv_t-\vw^k_t \right\|^2\right\},
\end{multline}
where $\vw^{\rm opt},\ \psi$ are defined in \eqref{eq:w_opt_def}, \eqref{eq:psi_heterogeneity_gap} respectively. The expression in \eqref{eq:expected_distance_bound} bounds the expected distance between the virtual sequence $\{\vv_t\}$ and the optimal weights $\vw^{\rm opt}$ in a recursive manner. We further bound the summands in \eqref{eq:expected_distance_bound}, using the following lemmas:

\begin{lemma}\label{lemma:g_functions_bound}
If the step-size $\eta_t$ is non-increasing and satisfies $\eta_t \leq 2\eta_{t+\tau}$ for each $t\geq 0$, then, when Assumption \ref{itm:bounded_norm} is satisfied, it holds that
 \begin{align}\label{eq:g_functions_bound}
    \eta^2_t \E\left\{\left\| \tilde\vg_t-\vg_t \right\|^2\right\} \leq \left(1+36\tau^2\sigma^2\right) 
    \eta_{t}^2\sum_{k=1}^K \alpha_k^2 \xi^2_k.
\end{align}
\end{lemma}
\begin{lemma}\label{lemma:v_w_functions_bound}
If the step-size $\eta_t$ is non-increasing and satisfies $\eta_t \leq 2\eta_{t+\tau}$ for each $t\geq 0$, then, when Assumption \ref{itm:bounded_norm}, it holds that
\begin{align}\label{eq:v_w_functions_bound}
    \E\left\{\sum_{k=1}^K\alpha_{k}\left\|\vv_t-\vw^k_t \right\|^2\right\} \leq 4(\tau-1)^2
    \eta_{t}^2\sum_{k=1}^K \alpha_k \xi^2_k,
\end{align}
\end{lemma}
The proof of Lemma \ref{lemma:v_w_functions_bound} is given in \cite[App. C]{shlezinger2020uveqfed}, and so that of Lemma \ref{lemma:g_functions_bound}, with $\zeta^2=9/M$ and ${\bar{\sigma}}^2_{\cL}=\sigma^2$.

Next, we define $\delta_t\triangleq\E\left\{\left\| \vv_t-\vw^{\rm opt} \right\|^2\right\}$. When $t=n\cdot\tau,\ n\in\N$, the term $\delta_t$ represents the $\ell_2$-norm of the error in the weights of the global model. Using Lemmas \ref{lemma:g_functions_bound}-\ref{lemma:v_w_functions_bound}, while integrating \eqref{eq:g_functions_bound} and \eqref{eq:v_w_functions_bound} into \eqref{eq:expected_distance_bound}, we obtain the following recursive relationship on the weights error:
\begin{align}\label{eq:recursive_delta}
    \delta_{t+1}&\leq(1-\eta_t\rho_c)\delta_t+\eta^2_tb, \text{ where }\\
    \notag b &\triangleq 
    \left(1+36\tau^2\sigma^2\right)\sum_{k=1}^K\alpha^2_k\xi^2_k 
    + 6\rho_s\psi
    + 8(\tau-1)^2\sum_{k=1}^K\alpha_k\xi^2_k.
\end{align}
The relationship in \eqref{eq:recursive_delta} is used in the sequel to prove the \ac{fl} convergence bound stated in Theorem \ref{thm:FL Convergence}.
\else
Recall that every $\tau$ local \ac{sgd} iterations, the $k$th user performs a transmission, that once done with \ac{jopeq} integration, its effect can be modeled as additional noise via $\tilde{\vh}^k_t=\vh^k_t+\text{err}_t^k$. 
Thus, $\text{err}_t^k = \boldsymbol 0$ when $t$ is an integer multiple of $\tau$ and otherwise, by Appendix \ref{app:jopeq_holds_tm}, each $L\times 1$ sub-vector of  $\text{err}_t^k$ scaled by $\zeta_t^k$ obeys the multivariate $t$-distribution $\mathbf{t}^d_\nu(\boldsymbol \mu,\mathbf\Sigma)$ which holds \eqref{eq:t_mech_eps}. The overall procedure can be compactly written~as
\begin{align*}
    \notag &\vw^k_{t+1}=\\
    &  \begin{cases*}
            \vw^k_t -\eta_{t}\nabla F_k^{i^k_{t}}\left(\vw^k_{t}\right)+\underbrace{\text{err}_t^k}_{=0} 
            \quad\text{ if } t+1 \neq n\cdot\tau,\ n\in\N,\\
            \sum_{k'=1}^K \alpha_{k'}\left(\vw^{k'}_t -\eta_{t}\nabla F_{k'}^{i^{k'}_t}\left(\vw^{k'}_{t}\right)+\text{err}_t^{k'} \right)
            \quad\text{ else}.
        \end{cases*}
\end{align*}

As in \cite{shlezinger2020uveqfed}, we next define $\vv_t \triangleq \sum_{k=1}^K \alpha_{k} \vw^k_t$, and the averaged noisy stochastic as well as the full gradients 
\begin{align}
    \tilde\vg_t &\triangleq \sum_{k=1}^K \alpha_{k}\left(\nabla F_{k}^{i^{k}_t}\left(\vw^{k}_{t}\right)-\frac{1}{\eta_t}\text{err}_t^{k} \right),\\
    \vg_t &\triangleq \sum_{k=1}^K \alpha_{k}\nabla F_{k}\left(\vw^{k}_{t}\right),
\end{align}
respectively. Note that for the specific choice of $\mathbf{t}^d_\nu(\boldsymbol 0,\mathbf\Sigma)$ \ac{jopeq}'s error is zero-mean, and it further holds that $\E\{\tilde\vg_t\}=\vg_t$ and $\vv_{t+1}=\vv_t - \eta_t\tilde\vg_t$.

The resulting model is thus equivalent to that used in \cite[App. C]{shlezinger2020uveqfed}, and therefore, by assumptions \ref{itm:obj_smooth}-\ref{itm:obj_convex} it follows that if $\eta_t\leq\frac{1}{4\rho_s}$ then
\begin{multline}\label{eq:expected_distance_bound}
    \E\left\{\left\| \vv_{t+1}-\vw^{\rm opt} \right\|^2\right\}\leq 
    (1-\eta_t\rho_c) \E\left\{\left\| \vv_t-\vw^{\rm opt} \right\|^2\right\} \\
     +6\rho_s\eta^2_t\psi
     + \eta^2_t \E\left\{\left\| \tilde\vg_t-\vg_t \right\|^2\right\} \\
    + 2\E\left\{\sum_{k=1}^K\alpha_{k}\left\|\vv_t-\vw^k_t \right\|^2\right\},
\end{multline}
where $\vw^{\rm opt},\ \psi$ are defined in \eqref{eq:w_opt_def}, \eqref{eq:psi_heterogeneity_gap} respectively. We further bound the summands in \eqref{eq:expected_distance_bound}, using Lemmas C.1 and C.2 in  \cite[App. C]{shlezinger2020uveqfed}
 \begin{align}
    \eta^2_t \E\left\{\left\| \tilde\vg_t-\vg_t \right\|^2\right\} \leq \left(1+36\tau^2\sigma^2\right) 
    \eta_{t}^2\sum_{k=1}^K \alpha_k^2 \xi^2_k, \label{eq:g_functions_bound}\\
    \E\left\{\sum_{k=1}^K\alpha_{k}\left\|\vv_t-\vw^k_t \right\|^2\right\} \leq 4(\tau-1)^2
    \eta_{t}^2\sum_{k=1}^K \alpha_k \xi^2_k. \label{eq:v_w_functions_bound}
\end{align}

Next, we define $\delta_t\triangleq\E\left\{\left\| \vv_t-\vw^{\rm opt} \right\|^2\right\}$. When $t=n\cdot\tau,\ n\in\N$, the term $\delta_t$ represents the $\ell_2$-norm of the error in the weights of the global model. Integrating \eqref{eq:g_functions_bound} and \eqref{eq:v_w_functions_bound} into \eqref{eq:expected_distance_bound}, we obtain the following recursive relationship on the weights error:
\begin{align}\label{eq:recursive_delta}
    \delta_{t+1}&\leq(1-\eta_t\rho_c)\delta_t+\eta^2_tb, \text{ where }\\
    \notag b &\triangleq 
    \left(1+36\tau^2\sigma^2\right)\sum_{k=1}^K\alpha^2_k\xi^2_k 
    + 6\rho_s\psi
    + 8(\tau-1)^2\sum_{k=1}^K\alpha_k\xi^2_k.
\end{align}
The relationship in \eqref{eq:recursive_delta} is used in the sequel to prove the \ac{fl} convergence bound stated in Theorem \ref{thm:FL Convergence}.
\fi
\subsubsection{FL Convergence Bound} 
\ifextended
Here, we prove Theorem \ref{thm:FL Convergence} based on the recursive relationship in \eqref{eq:recursive_delta}. This is achieved by properly  setting the step-size and the \ac{fl} systems parameters in \eqref{eq:recursive_delta} to bound $\delta_t:\E\left\{\left\| \vv_t-\vw^{\rm opt} \right\|^2\right\}$, and combining the resulting bound with the strong convexity of the objective \ref{itm:obj_convex} to prove \eqref{eq:FL_Convergence}.

In particular, we set the step-size $\eta_t$ to take the form $\eta_t=\frac{\beta}{t+\varphi}$ for some $\beta>0$ and $\varphi \geq \max (4\rho_s \beta, \tau)$, for which $\eta_t\leq\frac{1}{4\rho_s}$ and $\eta_t\leq 2 \eta_{t+\tau}$, implying that \eqref{eq:expected_distance_bound}, \eqref{eq:g_functions_bound} and \eqref{eq:v_w_functions_bound} hold. 
Under such settings, in  \cite[App. C]{shlezinger2020uveqfed} is it proved by induction that for $\lambda\geq\max\left(\frac{1+\beta^2b}{\beta\rho_c},\varphi\delta_0\right)$ is holds that $\delta_t\leq\frac{\lambda}{t+\varphi}$ for all integer $t\geq0$.
Finally, the smoothness of the objective \ref{itm:obj_smooth} implies that
\begin{align}\label{eq:smoothness_implies}
    \E\{F(\vw_t)-F(\vw^{\rm opt})\}\leq\frac{\rho_s}{2}\delta_t\leq\frac{\rho_2\lambda}{2(t+\varphi)}.
\end{align}
Setting $\beta=\frac{\tau}{\rho_c}$ results in $\varphi\geq\tau\max(1,4\rho_s/\rho_c)$ and $\lambda\geq\max(\frac{\rho^2_c+\tau^2b}{\tau\rho_c},\varphi\delta_0)$, once substituted into \eqref{eq:smoothness_implies}, proves \eqref{eq:FL_Convergence}.
\else
We set the step-size $\eta_t$ to take the form $\eta_t=\frac{\beta}{t+\varphi}$ for some $\beta>0$ and $\varphi \geq \max (4\rho_s \beta, \tau)$, for which $\eta_t\leq\frac{1}{4\rho_s}$ and $\eta_t\leq 2 \eta_{t+\tau}$, implying that \eqref{eq:expected_distance_bound}, \eqref{eq:g_functions_bound} and \eqref{eq:v_w_functions_bound} hold. 
Under such settings, in  \cite[App. C]{shlezinger2020uveqfed} is it proved by induction that for $\lambda\geq\max\left(\frac{1+\beta^2b}{\beta\rho_c},\varphi\delta_0\right)$, it holds that $\delta_t\leq\frac{\lambda}{t+\varphi}$ for all integer $t\geq0$.
Finally, the smoothness of the objective \ref{itm:obj_smooth} implies that
\begin{align}\label{eq:smoothness_implies}
    \E\{F(\vw_t)-F(\vw^{\rm opt})\}\leq\frac{\rho_s}{2}\delta_t\leq\frac{\rho_2\lambda}{2(t+\varphi)}.
\end{align}
Setting $\beta=\frac{\tau}{\rho_c}$ results in $\varphi\geq\tau\max(1,4\rho_s/\rho_c)$ and $\lambda\geq\max(\frac{\rho^2_c+\tau^2b}{\tau\rho_c},\varphi\delta_0)$, once substituted into \eqref{eq:smoothness_implies}, proves \eqref{eq:FL_Convergence}.
\fi
\fi
\end{appendix}

\bibliographystyle{IEEEtran}
\bibliography{IEEEabrv,bib}

\end{document}